\documentclass{article}
\usepackage[letterpaper, left=1in, right=1in, top=1in, bottom=1in]{geometry}

\usepackage{natbib}
\bibpunct{(}{)}{;}{a}{,}{,}
\usepackage[colorlinks,citecolor=blue!70!black,linkcolor=blue!70!black,
urlcolor=blue!70!black,breaklinks=true]{hyperref}

\usepackage{parskip}

\usepackage[utf8]{inputenc} 
\usepackage[T1]{fontenc}    
\usepackage{booktabs}       
\usepackage{microtype}      

\usepackage{amsmath, amsthm, amssymb}
\usepackage{dsfont} 
\usepackage{url}            
\usepackage{algorithm}
\usepackage[noend]{algorithmic}
\usepackage{listings}
\usepackage{bm}
\usepackage{bbm}
\usepackage{enumitem}
\usepackage{nicefrac}       
\usepackage{mathrsfs} 
\usepackage{inconsolata}
\usepackage{comment}
\usepackage{xspace}
\usepackage{transparent}
\usepackage[nameinlink, capitalize]{cleveref}
\makeatletter
\AddToHook{cmd/appendix/before}{
\def\cref@section@alias{appendix}
\def\cref@subsection@alias{appendix}
\def\cref@subsubsection@alias{appendix}
}
\makeatother
\usepackage{thmtools}
\usepackage{thm-restate}
\usepackage{nicematrix}
\usepackage{arydshln}
\usepackage{mathtools}
\usepackage{lineno}
\usepackage{xspace}

\usepackage{enumitem}
\setlist[enumerate]{leftmargin=.2in}
\setlist[itemize]{leftmargin=.2in}

\def\E{{\mathbb E}}

\def\P{{\mathbb P}}
\def\R{{\mathbb R}}
\def\N{{\mathbb N}}

\newcommand{\mfe}{\mathfrak{e}}
\newcommand{\mfs}{\mathfrak{s}}
\newcommand{\mfF}{\mathfrak{F}}

\DeclareMathOperator*{\argmin}{arg\,min}

\DeclareMathOperator{\err}{err}
\DeclareMathOperator{\poly}{poly}

\DeclareMathOperator{\unif}{{unif}}

\DeclareMathOperator{\hard}{{hard}}
\DeclareMathOperator{\easy}{{easy}}
\DeclareMathOperator{\flex}{{flexible}}

\DeclareMathOperator{\stan}{{standard}}
\DeclareMathOperator{\chow}{{Chow}}

\DeclareMathOperator{\1}{{\mathds{1}}}

\DeclareMathOperator{\val}{{val}}

\newcommand{\trn}{\top}

\newcommand{\polylog}{\mathrm{polylog}}
\renewcommand{\epsilon}{\varepsilon}

\newtheorem{theorem}{Theorem}
\newtheorem{lemma}{Lemma}

\newtheorem{definition}{Definition}
\newtheorem{assumption}{Assumption}
\newtheorem{proposition}{Proposition}

\newtheorem{remark}{Remark}
\newtheorem{example}{Example}

\newcommand{\algcommentlight}[1]{\textcolor{blue!70!black}{\transparent{0.5}\small{\texttt{\textbf{//\hspace{2pt}#1}}}}}

\DeclarePairedDelimiter{\abs}{\lvert}{\rvert} \DeclarePairedDelimiter{\brk}{[}{]}
\DeclarePairedDelimiter{\crl}{\{}{\}}
\DeclarePairedDelimiter{\prn}{(}{)}
\DeclarePairedDelimiter{\nrm}{\|}{\|}
\DeclarePairedDelimiter{\ang}{\langle}{\rangle}

\DeclarePairedDelimiter{\ceil}{\lceil}{\rceil}
\DeclarePairedDelimiter{\floor}{\lfloor}{\rfloor}

\DeclarePairedDelimiterX{\infdiv}[2]{(}{)}{#1\;\delimsize\|\;#2}

\newcommand{\wt}[1]{\widetilde{#1}}
\newcommand{\wh}[1]{\widehat{#1}}
\newcommand{\wb}[1]{\widebar{#1}}

\def\ddefloop#1{\ifx\ddefloop#1\else\ddef{#1}\expandafter\ddefloop\fi}
\def\ddef#1{\expandafter\def\csname bb#1\endcsname{\ensuremath{\mathbb{#1}}}}
\ddefloop ABCDEFGHIJKLMNOPQRSTUVWXYZ\ddefloop
\def\ddefloop#1{\ifx\ddefloop#1\else\ddef{#1}\expandafter\ddefloop\fi}
\def\ddef#1{\expandafter\def\csname b#1\endcsname{\ensuremath{\mathbf{#1}}}}
\ddefloop ABCDEFGHIJKLMNOPQRSTUVWXYZ\ddefloop
\def\ddef#1{\expandafter\def\csname sf#1\endcsname{\ensuremath{\mathsf{#1}}}}
\ddefloop ABCDEFGHIJKLMNOPQRSTUVWXYZ\ddefloop
\def\ddef#1{\expandafter\def\csname c#1\endcsname{\ensuremath{\mathcal{#1}}}}
\ddefloop ABCDEFGHIJKLMNOPQRSTUVWXYZ\ddefloop
\def\ddef#1{\expandafter\def\csname h#1\endcsname{\ensuremath{\widehat{#1}}}}
\ddefloop ABCDEFGHIJKLMNOPQRSTUVWXYZ\ddefloop
\def\ddef#1{\expandafter\def\csname hc#1\endcsname{\ensuremath{\widehat{\mathcal{#1}}}}}
\ddefloop ABCDEFGHIJKLMNOPQRSTUVWXYZ\ddefloop
\def\ddef#1{\expandafter\def\csname t#1\endcsname{\ensuremath{\widetilde{#1}}}}
\ddefloop ABCDEFGHIJKLMNOPQRSTUVWXYZ\ddefloop
\def\ddef#1{\expandafter\def\csname tc#1\endcsname{\ensuremath{\widetilde{\mathcal{#1}}}}}
\ddefloop ABCDEFGHIJKLMNOPQRSTUVWXYZ\ddefloop
\def\ddefloop#1{\ifx\ddefloop#1\else\ddef{#1}\expandafter\ddefloop\fi}
\def\ddef#1{\expandafter\def\csname scr#1\endcsname{\ensuremath{\mathscr{#1}}}}
\ddefloop ABCDEFGHIJKLMNOPQRSTUVWXYZ\ddefloop

\let\oldparagraph\paragraph
\renewcommand{\paragraph}[1]{\oldparagraph{#1.}}

\renewcommand{\epsilon}{\varepsilon}

\newcommand{\ind}{\mathbbm{1}}    
\newcommand{\eps}{\epsilon}

\newcommand{\ldef}{\vcentcolon=}
\newcommand{\rdef}{=\vcentcolon}

\renewcommand{\bigm}[1]{\ifcsname fenced@\string#1\endcsname
    \expandafter\@firstoftwo
  \else
    \expandafter\@secondoftwo
  \fi
  {\expandafter\amsmath@bigm\csname fenced@\string#1\endcsname}{\amsmath@bigm#1}}

\newcommand{\DeclareFence}[2]{\@namedef{fenced@\string#1}{#2}}
\makeatother

\makeatletter
\let\save@mathaccent\mathaccent
\newcommand*\if@single[3]{\setbox0\hbox{${\mathaccent"0362{#1}}^H$}\setbox2\hbox{${\mathaccent"0362{\kern0pt#1}}^H$}\ifdim\ht0=\ht2 #3\else #2\fi
  }
\newcommand*\rel@kern[1]{\kern#1\dimexpr\macc@kerna}
\newcommand*\widebar[1]{\@ifnextchar^{{\wide@bar{#1}{0}}}{\wide@bar{#1}{1}}}
\newcommand*\wide@bar[2]{\if@single{#1}{\wide@bar@{#1}{#2}{1}}{\wide@bar@{#1}{#2}{2}}}
\newcommand*\wide@bar@[3]{\begingroup
  \def\mathaccent##1##2{\let\mathaccent\save@mathaccent
\if#32 \let\macc@nucleus\first@char \fi
\setbox\z@\hbox{$\macc@style{\macc@nucleus}_{}$}\setbox\tw@\hbox{$\macc@style{\macc@nucleus}{}_{}$}\dimen@\wd\tw@
    \advance\dimen@-\wd\z@
\divide\dimen@ 3
    \@tempdima\wd\tw@
    \advance\@tempdima-\scriptspace
\divide\@tempdima 10
    \advance\dimen@-\@tempdima
\ifdim\dimen@>\z@ \dimen@0pt\fi
\rel@kern{0.6}\kern-\dimen@
    \if#31
      \overline{\rel@kern{-0.6}\kern\dimen@\macc@nucleus\rel@kern{0.4}\kern\dimen@}\advance\dimen@0.4\dimexpr\macc@kerna
\let\final@kern#2\ifdim\dimen@<\z@ \let\final@kern1\fi
      \if\final@kern1 \kern-\dimen@\fi
    \else
      \overline{\rel@kern{-0.6}\kern\dimen@#1}\fi
  }\macc@depth\@ne
  \let\math@bgroup\@empty \let\math@egroup\macc@set@skewchar
  \mathsurround\z@ \frozen@everymath{\mathgroup\macc@group\relax}\macc@set@skewchar\relax
  \let\mathaccentV\macc@nested@a
\if#31
    \macc@nested@a\relax111{#1}\else
\def\gobble@till@marker##1\endmarker{}\futurelet\first@char\gobble@till@marker#1\endmarker
    \ifcat\noexpand\first@char A\else
      \def\first@char{}\fi
    \macc@nested@a\relax111{\first@char}\fi
  \endgroup
}
\makeatother

\newcommand{\Reg}{\mathrm{\mathbf{Regret}}}

\newcommand{\AlgLcb}{\mathrm{\mathbf{Alg}}_{\mathsf{lcb}}}
\newcommand{\AlgUcb}{\mathrm{\mathbf{Alg}}_{\mathsf{ucb}}}
\newcommand{\ucb}{\mathsf{ucb}}
\newcommand{\lcb}{\mathsf{lcb}}

\newcommand{\pseud}{\mathrm{Pdim}}

\newcommand{\reg}{\Reg}

\newcommand{\exc}{\mathsf{excess}}

\newcommand{\curly}{\crl}
\newcommand{\paren}{\prn}

\newcommand{\sq}{\brk}
\newcommand{\pdim}{\pseud}

\title{Efficient Active Learning with Abstention}
\date{}
\author{
Yinglun Zhu\\
{\normalsize University of Wisconsin--Madison}\\
{\normalsize\texttt{yinglun@cs.wisc.edu}}
\and
\and
Robert Nowak\\
{\normalsize University of Wisconsin--Madison}\\
{\normalsize\texttt{rdnowak@wisc.edu}}
}

\begin{document}

\maketitle

\begin{abstract}
The goal of active learning is to achieve the same accuracy achievable by passive learning, while using much fewer labels. Exponential savings in terms of label complexity have been proved in very special cases, but fundamental lower bounds show that such improvements are impossible in general. This suggests a need to explore alternative goals for active learning. Learning with abstention is one such alternative.  In this setting, the active learning algorithm may abstain from prediction and incur an error that is marginally smaller than random guessing. We develop the first computationally efficient active learning algorithm with abstention. Our algorithm provably achieves $\mathsf{polylog}(\frac{1}{\varepsilon})$ label complexity, without any low noise conditions. Such performance guarantee reduces the label complexity by an exponential factor, relative to passive learning and active learning that is not allowed to abstain. Furthermore, our algorithm is guaranteed to only abstain on hard examples (where the true label distribution is close to a fair coin), a novel property we term \emph{proper abstention} that also leads to a host of other desirable characteristics (e.g., recovering minimax guarantees in the standard setting, and avoiding the undesirable ``noise-seeking'' behavior often seen in active learning). We also provide novel extensions of our algorithm that achieve \emph{constant} label complexity and deal with model misspecification.
\end{abstract}

\section{Introduction}

Active learning aims at learning an accurate classifier with a small number of labeled data points \citep{settles2009active, hanneke2014theory}. Active learning has become increasingly important in modern application of machine learning, where unlabeled data points are abundant yet the labeling process requires expensive time and effort. Empirical successes of active learning have been observed in many areas \citep{tong2001support, gal2017deep, sener2018active}.
In noise-free or certain low-noise cases (i.e., under Massart noise \citep{massart2006risk}), active learning algorithms with \emph{provable} exponential savings over the passive counterpart have been developed  \citep{balcan2007margin, hanneke2007bound,  dasgupta2009analysis, hsu2010algorithms, dekel2012selective, hanneke2014theory, zhang2014beyond, krishnamurthy2019active, katz2021improved}.
On the other hand, however, not much can be said in the general case. 
In fact, \citet{kaariainen2006active} provides a $\Omega(\frac{1}{\epsilon^2})$ lower bound by reducing active learning to a simple mean estimation problem: It takes $\Omega(\frac{1}{\epsilon^2})$ samples to distinguish $\eta(x) = \frac{1}{2} + \epsilon$ and $\eta(x) = \frac{1}{2} - \epsilon$. Even with the relatively benign Tsybakov noise \citep{tsybakov2004optimal}, \citet{castro2006upper, castro2008minimax} derive a $\Omega( \poly(\frac{1}{\epsilon}))$ lower bound, again, indicating that exponential speedup over passive learning is not possible in general. These fundamental lower bounds lay out statistical barriers to active learning, and suggests considering a refinement of the label complexity goals in active learning \citep{kaariainen2006active}.

Inspecting  these lower bounds, one can see that active learning suffers from classifying hard examples that are close to the decision boundary. However, \emph{do we really require a trained classifier to do well on those hard examples?} In high-risk domains such as medical imaging, it makes more sense for the classifier to abstain from making the decision and leave the problem to a human expert. Such idea is formalized under Chow's error \citep{chow1970optimum}: Whenever the classifier chooses to abstain, a loss that is barely smaller than random guessing, i.e., $\frac{1}{2} - \gamma$, is incurred. 
The parameter $\gamma$ should be thought as a small positive quantity, e.g., $\gamma = 0.01$.
The inclusion of abstention is not only practically interesting, but also provides a statistical refinement of the label complexity goal of active learning: Achieving exponential improvement under Chow's excess error.
When abstention is allowed as an action, \citet{puchkin2021exponential} shows, for the first time, that exponential improvement in label complexity can be achieved by active learning in the general setting.  
However, the approach provided in \citet{puchkin2021exponential} can not be efficiently implemented. 
Their algorithm follows the disagreement-based approach and requires maintaining a version space and checking whether or not an example lies in the region of disagreement.
It is not clear how to generally implement these operations besides enumeration \citep{beygelzimer2010agnostic}.
Moreover, their algorithm relies on an Empirical Risk Minimization (ERM) oracle, which is known to be NP-Hard even for a simple linear hypothesis class \citep{guruswami2009hardness}.

In this paper, we break the computational barrier and design an efficient active learning algorithm with exponential improvement in label complexity relative to conventional passive learning.
The algorithm relies on weighted square loss regression oracle, which can be efficiently implemented in many cases \citep{krishnamurthy2017active, krishnamurthy2019active, foster2018practical, foster2020instance}. The  algorithm  also abstains properly, i.e., abstain only when it is the optimal choice, which allows us to easily translate the guarantees to the \emph{standard} excess error.
Along the way, we propose new noise-seeking noise conditions and show that: ``uncertainty-based'' active learners can be easily trapped, yet our algorithm provably overcome these noise-seeking conditions.
As an extension, we also provide the first algorithm that enjoys \emph{constant} label complexity for a \emph{general} set of regression functions.

\subsection{Problem setting}
\label{sec:setting}

Let $\cX$ denote the input space and $\cY$ denote the label space. We focus on the binary classification problem where $\cY = \curly*{0, 1}$. The joint distribution over $\cX \times \cY$ is denoted as $\cD_{\cX \cY}$. We use $\cD_{\cX}$ to denote the marginal distribution over the input space $\cX$, and use $\cD_{\cY \vert x}$ to denote the conditional distribution of $\cY$ with respect to any $x \in \cX$.
We define $\eta(x) \ldef \P_{y \sim \cD_{\cY \vert x}} (y = 1)$ as the conditional probability of taking the label of $1$.
We consider the standard active learning setup where $(x,y) \sim \cD_{\cX \cY}$ but $y$ is observed only after a label querying.
We consider hypothesis class $\cH: \cX \rightarrow \cY$. For any classifier $h \in \cH$, its (standard) error is defined as $\err(h) \ldef \P_{(x,y) \sim \cD_{\cX \cY}} (h(x) \neq y)$.

\paragraph{Function approximation}
We focus on the case where the hypothesis class $\cH$ is induced from a set of regression functions $\cF: \cX \rightarrow [0,1]$ that predicts the conditional probability $\eta(x)$. We write $\cH = \cH_{\cF} \coloneqq \curly*{h_f : f \in \cF}$ where $h_f(x) \ldef \ind(f(x) \geq 1/2)$. 
The complexity of $\cF$ is measured by the well-known complexity measure: the \emph{Pseudo dimension} $\pseud(\cF)$ \citep{pollard1984convergence, haussler1989decision,haussler1995sphere}; 
we assume $\pseud(\cF) < \infty$ throughout the paper.\footnote{See \cref{app:concentration} for formal definition of the Pseudo dimension. Many function classes of practical interests have finite Pseudo dimension: (1) when $\cF$ is finite, we have $\pseud(\cF) = O(\log \abs{\cF})$; (2) when $\cF$ is a set of linear functions/generalized linear function with non-decreasing link function, we have  $\cF = O(d)$; (3) when $\cF$ is a set of degree-$r$ polynomial in $\R^{d}$, we have $\pseud(\cF) = O({d+r \choose r})$.
}
Following existing works in active learning \citep{dekel2012selective,krishnamurthy2017active, krishnamurthy2019active} and contextual bandits \citep{agarwal2012contextual, foster2018practical, foster2020beyond, simchi2020bypassing}, we make the following \emph{realizability} assumption.

\begin{assumption}[Realizability]
\label{asmp:predictable}
The learner is given a set of regressors $\cF: \cX \to [0, 1]$ such that there exists a $f^\star \in \cF$ characterize the true conditional probability, i.e., $f^\star = \eta$.
\end{assumption}

The realizability assumption allows \emph{rich function approximation}, which strictly generalizes the setting with linear function approximation studied in active learning (e.g., in \citep{dekel2012selective}).
We relax \cref{asmp:predictable} in \cref{sec:misspecified} to deal with model misspecification.

\paragraph{Regression oracle}
We consider a regression oracle over $\cF$, which is extensively studied in the literature in active learning and contextual bandits \citep{krishnamurthy2017active, krishnamurthy2019active, foster2018practical, foster2020instance}.
Given any set $\cS$ of weighted examples $(w, x, y) \in \R_+ \times \cX \times \cY$ as input, the regression oracle outputs 
\begin{align}
\label{eq:regression_oracle}	
    \widehat f = \argmin_{f \in \cF} \sum_{(w, x, y) \in \cS} w \paren*{f(x) - y}^2.
\end{align}
The regression oracle solves a convex optimization problem with respect to the regression function, and admits closed-form solutions in many cases, e.g., it is reduced to least squares when $f$ is linear.
We view the implementation of the regression oracle as an efficient operation and quantify the computational complexity in terms of the number of calls to the regression oracle.

\paragraph{Chow's excess error \citep{chow1970optimum}} 
Let $h^\star \ldef h_{f^{\star}} \in \cH$ denote the Bayes classifier.
The \emph{standard excess error} of classifier $h \in \cH$ is defined as $\err(h) - \err(h^\star)$. 
Since achieving exponential improvement (of active over passive learning) with respect to the standard excess error is impossible in general \citep{kaariainen2006active}, we introduce Chow's excess error next.
We consider classifier of the form $\widehat h: \cX \rightarrow \cY \cup \curly*{\bot}$ where $\bot$ denotes the action of abstention. For any fixed $0 < \gamma < \frac{1}{2}$, the Chow's error is defined as 
\begin{align}
    \err_{\gamma}(\widehat h)  \ldef \P_{(x,y) \sim \cD_{\cX \cY}} \prn{\widehat h (x) \neq y,  \widehat h(x) \neq \bot} + \prn*{{1}/{2} - \gamma} \cdot \P_{(x,y) \sim \cD_{\cX \cY}} \prn{\widehat h(x) = \bot}.
    \label{eq:chow_error}
\end{align}

The parameter $\gamma$ can be chosen as a small constant, e.g., $\gamma = 0.01$, to avoid excessive abstention: The price of abstention is only marginally smaller than random guess.
The \emph{Chow's excess error} is then defined as $\err_{\gamma}(\wh h) - \err(h^\star)$ \citep{puchkin2021exponential}.
For any fixed accuracy level $\epsilon >0$, we aim at constructing a classifier $\widehat h: \cX \rightarrow \cY \cup \curly*{\bot}$ with $\eps$ Chow's excess error and $\polylog(\frac{1}{\eps})$ label complexity.   
We also relate Chow's excess error to standard excess error in \cref{sec:standard_excess_error}.

\subsection{Why Chow's excess error helps learning?}
\label{sec:chow}

We study the simple case where $\cX = \curly*{x}$ to illustrate the benefits of learning under Chow's excess error.
In this setting, the active learning problem reduces to mean estimation of the conditional probability $\eta(x) \in [0,1]$.
In the following, we compare learning behavior under standard excess error, Chow's excess error, and Chow's excess error relative to the optimal abstaining classifier.

\begin{figure}[t]
    \centering
    \includegraphics[width=\linewidth]{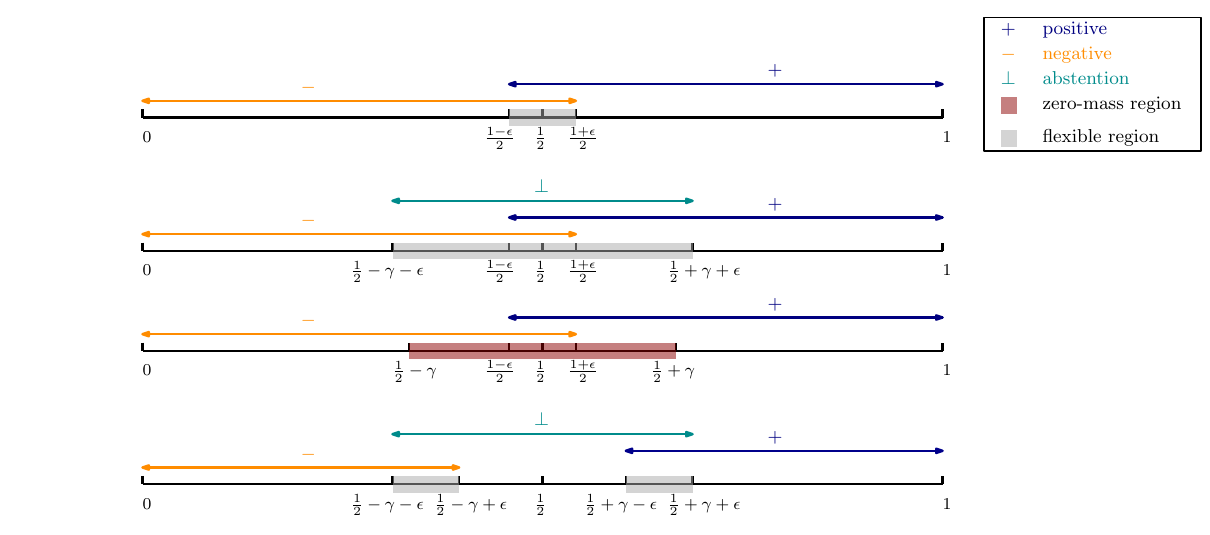}
    \caption{
    Illustration of decision regions under different error criteria.
    \emph{Top:} standard excess error $\err(\widehat h) - \err(h^\star)$. 
    \emph{Second:} Chow's excess error $\err_{\gamma}(\widehat h) - \err(h^\star)$. 
    \emph{Third:} standard excess error $\err(\widehat h) - \err(h^\star)$ under Massart noise condition with parameter $\gamma$. 
    \emph{Bottom:} Chow's excess error relative to the optimal abstaining classifier, i.e., $\err_{\gamma}(\wh h) -\inf_{h: \cX \rightarrow \crl{0, 1,\bot}} \err_{\gamma}(h)$. 
    In this figure, positive corresponds to predicting label 1 and negative to predicting label 0.
    }
    \label{fig:flexible_region}
\end{figure}

\paragraph{Learning under standard excess error}
Fix any $\epsilon > 0$. With respect to the conditional probability $ \eta(x)$, we define the \emph{positive region} $\cS_{+,\epsilon} \coloneqq \sq{\frac{1-\epsilon}{2}, 1}$ and the \emph{negative region} $\cS_{-,\epsilon} \coloneqq [0, \frac{1+\epsilon}{2}]$;
here, positive (resp. negative) refers to predicting label 1 (resp. 0).
These regions have the following interpretation:
if $\eta(x) \in \cS_{+,\epsilon}$ (resp. $\eta(x) \in \cS_{-,\eps}$), then labeling $x$ as 1 (resp. 0) incurs no more than $\epsilon$ standard excess error.
Under standard excess error, 
we define the \emph{flexible region} as $\cS_{\flex, \epsilon}^{\stan} \coloneqq \cS_{+, \epsilon} \cap \cS_{-, \epsilon} = [\frac{1-\epsilon}{2}, \frac{1+\epsilon}{2}]$, corresponding to the overlap of $\cS_{+, \eps}$ and $\cS_{-, \eps}$ (highlighted as the grey region in the top plot in \cref{fig:flexible_region}).
We have two key observations: (1) if $\eta(x) \in \cS_{\flex, \epsilon}^{\stan}$, then labeling $x$ as either 0 or 1 guarantees excess error at most $\epsilon$; and (2) if $\eta(x) \notin \cS_{\flex, \epsilon}^{\stan}$, achieving excess error at most $\eps$ requires correctly labeling $x$ as 0 or 1.
Since the flexible region has length $\eps$, it is possible to construct two learning scenarios where their $\eta(x)$ values differ by $O(\eps)$ yet require different labels. For instance, distinguishing between $\eta(x) = \frac{1}{2} - \eps$ and $\eta(x) = \frac{1}{2} + \eps$ yields a label complexity lower bound of $\Omega({1}/{\eps^2})$.

\paragraph{Learning under Chow's excess error}
We now consider learning under Chow's excess error.
As before, we define the positive and negative regions $\cS_{+,\epsilon} \coloneqq \sq{\frac{1-\epsilon}{2}, 1}$ and $\cS_{-,\epsilon} \coloneqq \sq{0, \frac{1+\epsilon}{2}}$.
Additionally, we introduce the \emph{abstention region}:
$\cS_{\bot, \epsilon} \ldef [\frac{1}{2}-\gamma-\epsilon, \frac{1}{2}+\gamma+\epsilon]$,
where abstaining on $x$ when $\eta(x) \in \cS_{\bot,\epsilon}$ incurs at most $\epsilon$ Chow's excess error.
Under Chow's excess error, the flexible region is enlarged thanks to the added abstention choice. We now have \emph{positive flexible region} $\cS_{\flex,+,\epsilon}^{\chow} \coloneqq \cS_{+,\epsilon} \cap \cS_{\bot,\epsilon} = \sq{\frac{1-\epsilon}{2}, \frac{1}{2} + \gamma + \epsilon}$, and \emph{negative flexible region} $\cS_{\flex,-,\epsilon}^{\chow} \coloneqq \cS_{-,\epsilon} \cap \cS_{\bot,\epsilon} = \sq{\frac{1}{2} - \gamma - \epsilon, \frac{1+\epsilon}{2}}$, both have length $\gamma + \frac{3 \eps}{2}$ (see the second plot in \cref{fig:flexible_region}).
These enlarged flexible regions imply that Chow's excess error can be controlled with fewer samples.
Specifically, it suffices to identify whether $\eta(x)$ lies within $\cS_{\bot,\epsilon}$ or confidently predicts $0/1$.
Constructing a confidence interval of length at most $\gamma/2$ requires $\wt O(1/\gamma^2)$ samples.
If $\eta(x) \in \sq{\frac{1 - \gamma}{2}, \frac{1 + \gamma}{2}}$, the confidence interval lies entirely within $\cS_{\bot,\epsilon}$, certifying the abstention acheives at most $\eps$ Chow's excess error.
If $\eta(x) < \frac{1 - \gamma}{2}$, the upper bound of the interval satisfies $\ucb(x) \le \frac{1}{2}$, certifying that labeling $x$ as $0$ achieves at most $\epsilon$ excess error.
Similarly, if $\eta(x) > \frac{1 + \gamma}{2}$, labeling $x$ as $1$ achieves at most $\epsilon$ excess error.
In summary, learning under Chow's excess error behaves similarly to learning under Massart noise (see the third plot in \cref{fig:flexible_region}).
Examples near the decision boundary are effectively filtered out by abstention, and reliable learning is achievable with $\wt O(1/\gamma^2)$ samples.

\oldparagraph{Why not compete against the optimal abstaining classifier?} 
We use $\err_{\gamma}(\wh{h}) - \inf_{h : \cX \to \{0, 1, \bot\}} \err_{\gamma}(h)$ 
to denote the excess error relative to the optimal classifier that is allowed to abstain.
As shown in the bottom plot of \cref{fig:flexible_region}, when competing against the optimal abstaining classifier, the flexible regions shrink back to length $O(\epsilon)$.
This occurs because abstention is the only action that guarantees at most $\epsilon$ excess error over the region $\prn{\frac{1}{2} - \gamma + \epsilon, \frac{1}{2} + \gamma - \epsilon}$.
Consequently, the learner must distinguish between cases like $\eta(x) = \frac{1}{2} + \gamma - 2 \epsilon$ and $\eta(x) = \frac{1}{2} + \gamma + 2 \epsilon$, which requires $\Omega(1 / \epsilon^2)$ samples.
Competing against the optimal abstaining classifier is also unreasonable.
For example, when $\eta(x) = \frac{1}{2} + \gamma - 2 \epsilon$, deciding whether to label $x$ as $1$ or abstain demands $\Omega(1 / \epsilon^2)$ samples.
Yet with only $\widetilde O(1 / \gamma^2)$ samples, the learner can already confidently determine that $\eta(x) > \frac{1}{2}$ and safely predict label $1$.

\subsection{Contributions and paper organization}

We provide informal statements of our main results in this section.
Our results depend on complexity measures such as \emph{value function} disagreement coefficient $\theta$ and eluder dimension $\mfe$ (formally defined in \cref{sec:epoch} and \cref{app:star_eluder}). 
These complexity measures are previously analyzed in contextual bandits \citep{russo2013eluder, foster2020instance} and we import them to the active learning setup.
These complexity measures are well-bounded for many function classes of practical interests, e.g., 
we have $\theta, \mfe = \wt O(d)$ for linear and generalized linear functions in $\R^d$.

Our first main contribution is that we design the first \emph{computationally efficient} active learning algorithm (\cref{alg:epoch}) that achieves exponential labeling savings, \emph{without any low noise assumptions}.

\begin{theorem}[Informal]
\label{thm:epoch_informal}
    There exists an algorithm that constructs a classifier $\wh h: \cX \rightarrow \crl*{0, 1, \bot}$ with Chow's excess error at most $\eps$ and label complexity $\wt O \prn{\frac{\theta \, \pseud(\cF)}{\gamma^2} \cdot \polylog(\frac{1}{\eps})}$, without any low noise assumptions. The algorithm can be efficiently implemented via a regression oracle: It takes $\wt O \prn{\frac{\theta \, \pseud(\cF)}{\eps \, \gamma^3}}$ oracle calls for general $\cF$, and $\wt O \prn{ \frac{\theta \, \pseud(\cF)}{\eps \, \gamma} }$ oracle calls for convex $\cF$.
\end{theorem}

The formal statements are provided in \cref{sec:epoch}.
The \emph{statistical} guarantees (i.e., label complexity) in \cref{thm:epoch_informal} is similar to the one achieved in \citet{puchkin2021exponential}, with one critical difference: The label complexity provided in \citet{puchkin2021exponential} is in terms of the \emph{classifier-based} disagreement coefficient $\check \theta$ \citep{hanneke2014theory}. 
Even for a set of linear classifier, $\check \theta$ is only known to be bounded in special cases, e.g., when $\cD_\cX$ is uniform over the unit sphere \citep{hanneke2007bound}.
On the other hand, we have  $\theta \leq d$ for any  $\cD_\cX$ \citep{foster2020instance}.

We say that a classifier $\wh h : \cX \rightarrow \crl{0, 1,\bot}$ enjoys proper abstention if it abstains only if abstention is indeed the optimal choice (based on \cref{eq:chow_error}).
For any classifier that enjoys proper abstention, one can easily relate its \emph{standard} excess error to the Chow's excess error, under commonly studied Massart/Tsybakov noises \citep{massart2006risk, tsybakov2004optimal}.
The classifier obtained in \cref{thm:epoch_informal} enjoys proper abstention, and achieves the following guarantees (formally stated in \cref{sec:minimax}).

\begin{theorem}[Informal]
\label{thm:proper_abs_informal}
Under Massart/Tsybakov noise, with appropriate adjustments,
the classifier learned in \cref{thm:epoch_informal} achieves the minimax optimal label complexity under standard excess error.
\end{theorem}

We also propose new noise conditions that \emph{strictly} generalize the usual Massart/Tsybakov noises, which we call noise-seeking conditions.
At a high-level, the noise-seeking conditions allow abundant data points with $\eta(x)$ equal/close to $\frac{1}{2}$. These points are somewhat ``harmless'' since it hardly matters what label is predicted at that point (in terms of excess error).
These seemingly ``harmless'' data points can, however, cause troubles for any active learning algorithm that requests the label for any point that is uncertain, i.e., the algorithm cannot decide if $\abs{\eta(x)-\frac{1}{2}}$ is strictly greater than $0$. We call such algorithms ``uncertainty-based'' active learners.
These algorithms could wastefully sample in these ``harmless'' regions, ignoring other regions where erring could be much more harmful.
We derive the following 
proposition (formally stated in \cref{sec:noise_seeking}) under these noise-seeking conditions.

\begin{proposition}[Informal]
\label{prop:budget_informal}
For any labeling budget $B \gtrsim \frac{1}{\gamma^2} \cdot \polylog(\frac{1}{\eps}) $, there exists a learning problem such that (1) any uncertainty-based active learner suffers standard excess error $\Omega(B^{-1})$; yet (2) the classifier $\wh h$ learned in \cref{thm:epoch_informal} achieves standard excess error at most $\eps$.
\end{proposition}

The above result demonstrates the superiority of our algorithm over any ``uncertainty-based'' active learner.
Moreover, we show that, under these strictly harder noise-seeking conditions, our algorithm still achieve guarantees similar to the ones stated in \cref{thm:proper_abs_informal}.

Before presenting our next main result, we first consider a simple active learning problem with $\cX=\crl{x}$.
Under Massart noise, we have $\abs{\eta(x) - \frac{1}{2}} \geq \tau_0$ for some constant $\tau_0 >0$. Thus, it takes no more than $O(\tau_0^{-2} \log\frac{1}{\delta})$ labels to achieve $\eps$ standard excess error, no matter how small $\eps$ is. 
This example shows that, at least in simple cases, we can expect to achieve a \emph{constant} label complexity for active learning, with no dependence on $\frac{1}{\eps}$ at all.
To the best of our knowledge,
our next result provides the first generalization of such phenomenon to a \emph{general} set of (finite) regression functions, as long as its eluder dimension $\mfe$ is bounded.
\begin{theorem}[Informal]
\label{thm:eluder_informal}
    Under Massart noise with parameter $\tau_0$ and a general (finite) set of regression function $\cF$. There exists an algorithm that returns a classifier with standard excess error at most $\eps$ and label complexity $O\paren{ \frac{\mfe \cdot \log ({\abs*{\cF}}/{\delta})}{\tau_0^2}}$, which is independent of $\frac{1}{\eps}$.
\end{theorem}

A similar constant label complexity holds with Chow's excess error, without any low noise assumptions.
We also provide discussion on why previous algorithms do not achieve such constant label complexity, even in the case with linear functions.
We defer formal statements and discussion to \cref{sec:constant}.
In \cref{sec:misspecified}, we relax \cref{asmp:predictable} and propose an algorithm that can deal with model misspecification.

\paragraph{Paper organization}
The rest of this paper is organized as follows.
We discuss additional related work in \cref{sec:related_work}.
We present our main algorithm and its guarantees in \cref{sec:epoch}. 
In \cref{sec:standard_excess_error}, we analyze our algorithm under standard excess error and discuss other key properties.
Extensions of the algorithm, including achieving \emph{constant} label complexity and handling model misspecification, are presented in \cref{sec:extension}.
Additional definitions and all proofs are deferred to the appendix.

\subsection{Additional related work}
\label{sec:related_work}

Learning under Chow's excess error is closely related to learning under Massart noise \citep{massart2006risk}, which assumes that no data point has conditional expectation close to the decision boundary,
i.e., $\P \prn*{ \abs{\eta(x) - 1 / 2} \leq \tau_{0} } = 0$ for some constant $\tau_0> 0$.
Learning under Massart noise is commonly studied in active learning \citep{balcan2007margin, hanneke2014theory, zhang2014beyond, krishnamurthy2019active}, where $\wt O(\tau_0^{-2})$ type of guarantees are achieved. 
Instead of making explicit assumptions on the underlying distribution, learning with Chow's excess error empowers the learner with the ability to abstain:
There is no need to make predictions on hard data points that are close to the decision boundary, 
i.e., $\crl{x: \abs{ \eta(x) - 1 / 2} \leq \gamma }$. 
Learning under Chow's excess error thus works on more general settings and still enjoys the 
$\wt O(\gamma^{-2})$ type of guarantee as learning under Massart noise \citep{puchkin2021exponential}.\footnote{However, passive learning with abstention only achieves error rate $\frac{1}{n \gamma}$ with $n$ samples \citep{bousquet2021fast}.}
We show in \cref{sec:standard_excess_error} that statistical guarantees achieved under Chow's excess error can be directly translated to guarantees under (usual and more challenging versions of) Massart/Tsybakov noise \citep{massart2006risk, tsybakov2004optimal}.

Active learning at aim competing the best in-class classifier with few labels. 
A long line of work directly works with the set of classifiers \citep{balcan2007margin, hanneke2007bound, hanneke2014theory, huang2015efficient, puchkin2021exponential}, where the algorithms are developed with (in general) hard-to-implement ERM oracles \citep{guruswami2009hardness} and the the guarantees dependence on the so-called disagreement coefficient \citep{hanneke2014theory}.
More recently, learning with function approximation have been studied inactive learning and contextual bandits \citep{dekel2012selective, agarwal2012contextual, foster2018practical, krishnamurthy2019active}.
The function approximation scheme permits efficient regression oracles, which solve convex optimization problems with respect to regression functions \citep{krishnamurthy2017active,krishnamurthy2019active,foster2018practical}.
It can also be analyzed with the scale-sensitive version of disagreement coefficient, which is usually tighter than the original one \citep{foster2020instance, russo2013eluder}.
Our algorithms are inspired \citet{krishnamurthy2019active}, where the authors study active learning under the standard excess error. 
The main deviation from \citet{krishnamurthy2019active} is that we need to \emph{manually} construct a classifier $\wh h$ with an abstention option and $\wh h \notin \cH$, which leads to differences in the analysis of excess error and label complexity.
We borrow techniques developed in contextual bandits \citet{russo2013eluder, foster2020instance} to analyze our algorithm.

Although one can also apply our algorithms in the nonparametric regime with proper pre-processing schemes such discretizations, our algorithm primarily works in the parametric setting with finite pseudo dimension \citep{haussler1995sphere} and finite (value function) disagreement coefficient \citep{foster2020instance}.
Active learning has also been studied in the nonparametric regime \citep{castro2008minimax,koltchinskii2010rademacher, minsker2012plug, locatelli2017adaptivity}.
Notably, \citet{shekhar2021active} studies Chow's excess error with margin-type of assumptions. 
Their setting is different to ours and $\poly(\frac{1}{\eps})$ label complexities are achieved.
If abundant amounts of data points are allowed to be exactly at the decision boundary, i.e., $\eta(x) = \frac{1}{2}$,
\citet{kpotufe2021nuances} recently shows that, in the nonparametric regime, no active learner can outperform the passive counterpart.

\section{Efficient active learning with abstention}
\label{sec:epoch}
We provide our main algorithm (\cref{alg:epoch}) in this section. \cref{alg:epoch} is an adaptation of the algorithm developed in \citet{krishnamurthy2017active, krishnamurthy2019active}, which studies active learning under the standard excess error (and Massart/Tsybakov noises).
We additionally take the abstention option into consideration, and \emph{manually construct} classifiers using the active set of (uneliminated) regression functions (which do not belong to the original hypothesis class).
These new elements allow us to achieve $\eps$ Chow's excess error with $\polylog(\frac{1}{\eps})$ label complexity, without any low noise assumptions.

\begin{algorithm}[]
	\caption{Efficient Active Learning with Abstention}
	\label{alg:epoch} 
	\renewcommand{\algorithmicrequire}{\textbf{Input:}}
	\renewcommand{\algorithmicensure}{\textbf{Output:}}
	\newcommand{\algorithmicbreak}{\textbf{break}}
    \newcommand{\BREAK}{\STATE \algorithmicbreak}
	\begin{algorithmic}[1]
		\REQUIRE Accuracy level $\eps >0 $, abstention parameter $\gamma \in (0, 1/2)$ and confidence level $\delta \in (0, 1)$.
		\STATE 
Define $T \ldef \wt O \prn{\frac{\theta \, \pseud(\cF)}{\eps \, \gamma}}$,
		$M \ldef \ceil{\log_2 T}$ and $C_\delta \ldef O \prn*{\pseud(\cF) \cdot \log( T/\delta)}$.
		
		\STATE Define $\tau_m \ldef 2^m$ for $m\geq1$, $\tau_0 \ldef 0$ and $\beta_m \ldef (M-m+1) \cdot C_\delta$. 
		\FOR{epoch $m = 1, 2, \dots, M$}
		\STATE Get $\widehat f_m \ldef \argmin_{f \in \cF} \sum_{t=1}^{\tau_{m-1}} Q_t \paren{f(x_t) - y_t}^2 $.\\
		\hfill \algcommentlight{We use $Q_t \in \crl{0,1}$ to indicate whether the label of  $x_t$ is queried.}
		\STATE (Implicitly) Construct active set of regression functions $\cF_m \subseteq \cF$ as
		\begin{align*}
		    \cF_m \ldef \crl*{ f \in \cF:  \sum_{t = 1}^{\tau_{m-1}} Q_t \prn*{f(x_t) - y_t}^2 \leq \sum_{t = 1}^{\tau_{m-1}} Q_t \paren{\widehat f_m(x_t) - y_t}^2 + \beta_m }. 
\end{align*}
		\STATE Construct classifier $\wh h_m: \cX \rightarrow \crl{0, 1,\bot}$ as 
		\begin{align*}
			\wh h_m (x) \ldef 
			\begin{cases}
				\bot, & \text{ if } \brk { \lcb(x;\cF_m), \ucb(x;\cF_m)} \subseteq 
				\brk*{ \frac{1}{2} - \gamma, \frac{1}{2} + \gamma}; \\
        \ind(\wh f_m(x) \geq \frac{1}{2} ) , & \text{o.w.}
			\end{cases}
\end{align*}
and construct query function $g_m(x)\ldef \ind \prn*{ \frac{1}{2} \in \prn{\lcb(x;\cF_m), \ucb(x;\cF_m)} } \cdot 
		\ind \prn{\wh h_m(x) \neq \bot}$.
		\IF{epoch $m=M$}
		\STATE \textbf{Return} classifier $\wh h_M$.
		\ENDIF
		\FOR{time $t = \tau_{m-1} + 1 ,\ldots , \tau_{m} $} 
		\STATE Observe $x_t \sim \cD_{\cX}$. Set $Q_t \ldef g_m(x_t)$.
		\IF{$Q_t = 1$}
		\STATE Query the label $y_t$ of $x_t$.
		\ENDIF
		\ENDFOR
		\ENDFOR

	\end{algorithmic}
\end{algorithm}

\cref{alg:epoch} runs in epochs of geometrically increasing lengths.
At the beginning of epoch $m \in [M]$, \cref{alg:epoch} first computes the empirical best regression function $\wh f_m$ that achieves the smallest cumulative square loss over previously labeled data points ($\wh f_1$ can be selected arbitrarily); it then (implicitly) constructs an active set of regression functions $\cF_m$, where the cumulative square loss of each $f \in \cF_m$ is not too much larger than the cumulative square loss of empirical best regression function $\wh f_m$. 
For any $x \in \cX$, based on the active set of regression functions, \cref{alg:epoch} constructs a lower bound $\lcb(x;\cF_m) \ldef \inf_{f \in \cF_m} f(x)$ and an upper bound $\ucb(x;\cF_m) \ldef \sup_{f \in \cF_m} f(x)$ for the true conditional probability $\eta(x)$.
An empirical classifier $\wh h_m: \cX \rightarrow \crl{0, 1, \bot}$ and a query function $g_m:\cX \rightarrow \crl{0, 1}$ are then constructed based on these confidence ranges and the abstention parameter $\gamma$.
For any time step $t$ within epoch $m$, \cref{alg:epoch} queries the label of the observed data point $x_t$ if and only if $Q_t \ldef g_m(x_t) = 1$.
\cref{alg:epoch} returns $\wh h_M$ as the learned classifier.

We now discuss the empirical classifier $\wh h_m$ and the query function $g_m$ in more detail.
Consider the event where $f^\star \in \cF_m$ for all $m \in [M]$, which can be shown to hold with high probability.
The constructed confidence intervals are valid under this event, i.e., $\eta(x) \in [\lcb(x;\cF_m), \ucb(x;\cF_m)]$.
First, let us examine the conditions that determine a label query.  The label of $x$ is \emph{not} queried if
\begin{itemize}
	\item \textbf{Case 1: $\wh h_m(x) = \bot$.} We have $\eta(x) \in \brk{ \lcb(x;\cF_m) , \ucb(x;\cF_m)} \subseteq \brk{ \frac{1}{2} - \gamma, \frac{1}{2}+\gamma}$.
Abstention leads to the smallest error \citep{herbei2006classification}, and no query is needed.
  \item \textbf{Case 2: $\frac{1}{2} \notin ( \lcb(x;\cF_m) , \ucb(x;\cF_m))$.} We have $\ind(\wh f_m(x)\geq \frac{1}{2}) = \ind(f^{\star}(x)\geq \frac{1}{2})$. Thus, no excess error is incurred and there is no need to query.
\end{itemize}
The only case when label query \emph{is} issued, and thus when the classifier $\wh h_m$ may suffer from excess error, is when 
\begin{align}
\frac{1}{2} \in ( \lcb(x;\cF_m) , \ucb(x;\cF_m)) \quad \text{and} \quad 
\brk*{ \lcb(x;\cF_m) , \ucb(x;\cF_m)} \nsubseteq \brk*{ \frac{1}{2} - \gamma, \frac{1}{2}+\gamma} 
\label{eq:epoch_query}
\end{align}
hold simultaneously. 
\cref{eq:epoch_query} necessarily leads to the condition $w(x;\cF_m) \ldef \ucb(x;\cF_m) - \lcb(x;\cF_m) > {\gamma}  $.
Our theoretical analysis shows that the event must $\ind( w(x;\cF_m) > \gamma)$ happens infrequently, and its frequency is closely related to the so-called \emph{value function disagreement coefficient} \citep{foster2020instance}, which we introduce as follows.\footnote{Compared to the original definition studied in contextual bandits \citep{foster2020instance}, our definition takes an additional ``sup'' over all possible marginal distributions $\cD_\cX$ to account for \emph{distributional shifts} incurred by selective querying (which do not occur in contextual bandits). Nevertheless, as we show below, our disagreement coefficient is still well-bounded for many important function classes.}
\begin{definition}[Value function disagreement coefficient]
    \label{def:dis_coeff}
    For any $f^{\star} \in \cF$ and $\gamma_0, \eps_0 > 0$,
    the value function disagreement coefficient $	 \theta^{\val}_{f^{\star}}(\cF, \gamma_0, \eps_0)$ is defined as 
    \begin{align*}
 \sup_{\cD_\cX}\sup_{\gamma> \gamma_0, \eps> \eps_0} 
	 \crl*{ \frac{\gamma^2}{\epsilon^2} \cdot 
	\P_{\cD_\cX} \prn*{ \exists f \in \cF: \abs{f(x) - f^{\star}(x)} > \gamma,
	\nrm*{ f - f^{\star}}_{\cD_\cX} \leq \eps} } \vee 1,
    \end{align*}
    where $\nrm{f}^2_{\cD_\cX} \ldef \E_{x \sim \cD_\cX} \brk{f^2(x)}$.
\end{definition}

Combining the insights discussed above, we derive the following label complexity guarantee for \cref{alg:epoch} (we use $\theta \ldef \sup_{f^{\star} \in \cF, \iota > 0} \theta^{\val}_{f^{\star}}(\cF,\gamma / 2, \iota)$ and discuss its boundedness below).
\footnote{It suffices to take $\theta \ldef \theta^{\val}_{f^{\star}}(\cF, \gamma /2, \iota)$ with $\iota \propto \sqrt{\gamma \eps}$ to derive a slightly different guarantee. See \cref{app:epoch}.}
\begin{restatable}{theorem}{thmEpoch}
\label{thm:epoch}
With probability at least $1-2\delta$, \cref{alg:epoch} returns a classifier with Chow's excess error at most $\epsilon$ and label complexity 
$O\prn{ \frac{\theta \, \pseud(\cF)}{\gamma^2} \cdot \log^2 \prn{\frac{\theta \, \pseud(\cF)}{\eps \, \gamma}}  \cdot  {\log\prn{\frac{\theta \, \pseud(\cF)}{\eps \, \gamma \,\delta}}}}$.
\end{restatable}

\cref{thm:epoch} shows that \cref{alg:epoch} achieves exponential label savings (i.e., $\polylog(\frac{1}{\eps})$) without any low noise assumptions. 
We discuss the result in more detail next.
\begin{itemize}
	\item \textbf{Boundedness of $\theta$.}
	The value function disagreement coefficient is well-bounded for many function classes of practical interests.
For instance, we have $\theta \leq d$ for linear functions on $\R^d$ and $\theta \leq C_{\textsf{link}} \cdot d$ for generalized linear functions (where $C_{\textsf{link}}$ is a quantity related to the link function).
Moreover, $\theta$ is \emph{always} upper bounded by complexity measures such as (squared) star number and eluder dimension \citep{foster2020instance}. See \cref{app:star_eluder} for the detailed definitions/bounds. 
\item \textbf{Comparison to \citet{puchkin2021exponential}.}
The label complexity bound derived in \cref{thm:epoch} is similar to the one derived in \citet{puchkin2021exponential}, with one critical difference: The bound derived in \citet{puchkin2021exponential} is in terms of \emph{classifier-based} disagreement coefficient $\check \theta$ \citep{hanneke2014theory}. Even in the case with linear classifiers, $\check \theta$ is only known to be bounded under additional assumptions, e.g., when $\cD_\cX$ is uniform over the unit sphere.
\end{itemize}

\paragraph{Computational efficiency} 
We discuss how to efficiently implement \cref{alg:epoch} with the regression oracle defined in \cref{eq:regression_oracle}.
\footnote{Recall that the implementation of the regression oracle should be viewed as an efficient operation since it solves a convex optimization problem with respect to the regression function, and it even admits closed-form solutions in many cases, e.g., it is reduced to least squares when $f$ is linear. On the other hand, the ERM oracle used in \citet{puchkin2021exponential} is NP-hard even for a set of linear classifiers \citep{guruswami2009hardness}.}
Our implementation relies on subroutines developed in \citet{krishnamurthy2017active, foster2018practical}, which allow us to approximate confidence bounds $\ucb(x;\cF_m)$ and $\lcb(x;\cF_m)$ up to $\alpha$ approximation error with $O(\frac{1}{\alpha^{2}} \log \frac{1}{\alpha})$ (or $O(\log \frac{1}{\alpha})$ when $\cF$ is convex and closed under pointwise convergence) calls to the regression oracle.
To achieve the same theoretical guarantees shown in \cref{thm:epoch} (up to changes in constant terms),
we show that it suffices to (i) control the approximation error at level $O(\frac{\gamma}{\log T})$, (ii) construct the approximated confidence bounds $\wh \lcb(x;\cF_m)$ and  $\wh \ucb(x;\cF_m)$ in a way such that the confidence region is  non-increasing with respect to the epoch $m$, i.e., $(\wh \lcb(x; \cF_m), \wh \ucb(x;\cF_m)) \subseteq (\wh \lcb(x; \cF_{m-1}), \wh \ucb(x;\cF_{m-1}))$ (this ensures that the sampling region is non-increasing even with \emph{approximated} confidence bounds, which is important to our theoretical analysis), and (iii) use the approximated confidence bounds $\wh \lcb(x;\cF_m)$ and  $\wh \ucb(x;\cF_m)$ to construct the classifier $\wh h_m$ and the query function $g_m$.
We provide our guarantees as follows, and leave details to \cref{app:epoch}
(we redefine $\theta \ldef \sup_{f^{\star} \in \cF, \iota > 0} \theta^{\val}_{f^{\star}}(\cF,\gamma / 4, \iota)$ in the \cref{thm:epoch_efficient} to account to approximation error).
\begin{restatable}{theorem}{thmEpochEfficient}
	\label{thm:epoch_efficient}
	\cref{alg:epoch} can be efficiently implemented via the regression oracle and enjoys the same theoretical guarantees stated in \cref{thm:epoch}.
	The number of oracle calls needed is $\wt O(\frac{\theta \, \pseud(\cF)}{\eps \, \gamma^{3}})$ for a general set of regression functions $\cF$, and $\wt O(\frac{\theta \, \pseud(\cF)}{\eps \, \gamma})$ when $\cF$ is convex and closed under pointwise convergence.
	The per-example inference time of the learned $\wh h_{M}$ is $\wt O ( \frac{1}{\gamma^2} \log^2 \prn{\frac{\theta \, \pseud(\cF)}{\eps }})$ for general $\cF$, and $\wt O ( \log \frac{1}{\gamma}) $ when $\cF$ is convex and closed under pointwise convergence.
\end{restatable}

With \cref{thm:epoch_efficient}, we provide the first computationally efficient active learning algorithm that achieves exponential label savings, without any low noise assumptions.

\section{Guarantees under standard excess error}
\label{sec:standard_excess_error}
We provide guarantees for \cref{alg:epoch} under \emph{standard} excess error. 
In \cref{sec:minimax}, we show that \cref{alg:epoch} can be used to recover the usual minimax label complexity under Massart/Tsybakov noise; we also provide a new learning paradigm based on \cref{alg:epoch} under limited budget.
In \cref{sec:noise_seeking}, we show that \cref{alg:epoch} provably avoid the undesired \emph{noise-seeking} behavior often seen in active learning.

\subsection{Recovering minimax optimal label complexity}
\label{sec:minimax}

One way to convert an abstaining classifier $\widehat h: \cX \rightarrow \cY \cup \curly*{\bot}$ into a standard classifier $\check h: \cX \rightarrow \cY$ is by randomizing the prediction in its abstention region, i.e., if $\wh h(x) = \bot$, then its randomized version $\check h(x)$ predicts $0$ and $1$ with equal probability \citep{puchkin2021exponential}. With such randomization, the \emph{standard excess error} of $\check h$ can be characterized as   
\begin{align}
	\label{eq:randomization}
    \err(\check h) - \err(h^\star) = \err_{\gamma}(\widehat h) - \err(h^\star) + \gamma \cdot \P_{x \sim \cD_{\cX }} (\widehat h(x) = \bot).
\end{align}
The standard excess error depends on the (random) abstention region of $\wh h$, which is difficult to quantify in general. 
To give a more practical characterization of the standard excess error, we introduce the concept of proper abstention in the following. 

\begin{definition}[Proper abstention]
\label{def:proper_abstention}
A classifier $\widehat h : \cX \rightarrow \cY \cup \curly*{\bot}$ enjoys proper abstention if and only if it abstains in regions where abstention is indeed the optimal choice, i.e., 
$\crl[\big]{x \in \cX: \widehat h(x) = \bot} \subseteq \crl*{x \in \cX: \eta(x) \in \brk*{\frac{1}{2} - \gamma , \frac{1}{2} + \gamma  } } \rdef \cX_\gamma$.
\end{definition}

\begin{restatable}{proposition}{propProperAbs}
\label{prop:proper_abstention}
The classifier $\wh h$ returned by \cref{alg:epoch} enjoys proper abstention. With randomization over the abstention region, we have the following upper bound on its standard excess error
\begin{align}
	\label{eq:prop_abstention}
    \err(\check h) - \err(h^\star)  
     \leq \err_{\gamma}(\widehat h) - \err(h^\star) + \gamma \cdot \P_{x \sim \cD_{\cX}} (x \in \cX_{\gamma}).
\end{align}
\end{restatable}

The proper abstention property of $\wh h$ returned by \cref{alg:epoch} is achieved via conservation: $\wh h$ will avoid abstention unless it is absolutely sure that abstention is the optimal choice.\footnote{On the other hand, however, the algorithm provided in \citet{puchkin2021exponential} is very unlikely to have such property. In fact, only a small but \emph{nonzero} upper bound of abstention rate is provided (Proposition 3.6 therein) under the Massart noise with $\gamma \leq \frac{\tau_0}{2}$; yet any classifier that enjoys proper abstention should have exactly zero abstention rate.}
To characterize the standard excess error of classifier with proper abstention, we only need to upper bound the term $ \P_{x \sim \cD_{\cX}} (x \in \cX_{\gamma})$, which does \emph{not} depends on the (random) classifier $\wh h$. Instead, it only depends on the marginal distribution. 
We next introduce the common Massart/Tsybakov noise conditions.
\begin{definition}[Massart noise, \citet{massart2006risk}]
	\label{def:massart}
  A distribution $\cD_{\cX \cY}$ satisfies the Massart noise condition with parameter $\tau_0> 0$ if
  $\P_{x \sim \cD_\cX} \paren*{ \abs*{\eta(x) - 1 / 2} \leq \tau_0} = 0$.
\end{definition}
\begin{definition}[Tsybakov noise, \citet{tsybakov2004optimal}]
	\label{def:tsybakov}
  A distribution $\cD_{\cX \cY}$ satisfies the Tsybakov noise condition with parameter $\beta \geq 0$ and a universal constant $c>0$ if
  $\P_{x \sim \cD_\cX} \paren*{\abs*{\eta(x) - 1 / 2} \leq  \tau} \leq c \, \tau^{\beta}$ for any $\tau > 0$.
\end{definition}

As in \citet{balcan2007margin, hanneke2014theory}, we assume knowledge of noise parameters (e.g., $\tau_0, \beta$).
Together with the active learning lower established in \citet{castro2006upper, castro2008minimax}, and focusing on the dependence of $\eps$, our next theorem shows that \cref{alg:epoch} can be used to recover the minimax label complexity in active learning, under the \emph{standard} excess error.

\begin{restatable}{theorem}{thmStandardExcessError}
\label{thm:standard_excess_error}
With an appropriate choice of the abstention parameter $\gamma$ in \cref{alg:epoch} and randomization over the abstention region, \cref{alg:epoch} learns a classifier $\check h$ at the minimax optimal rates: To achieve $\eps$ standard excess error, it takes $\wt \Theta(\tau_0^{-2})$ labels under Massart noise and takes $\wt \Theta \prn{ {\eps}^{ - 2 / (1 + \beta)} }$ labels under Tsybakov noise.
\end{restatable}

\begin{remark}
    In addition to recovering the minimax rates, the proper abstention property is desirable in practice:
It guarantees that $\wh h$ will not abstain on easy examples, i.e., it will not mistakenly flag easy examples as ``hard-to-classify'', thus eliminating unnecessary human labeling efforts.
\end{remark}

\cref{alg:epoch} can also be used to provide new learning paradigms in the limited budget setting, which we introduce below.
No prior knowledge of noise parameters are required in this setup.

\paragraph{New learning paradigm under limited budget} 
Given any labeling budget $B>0$, we can then choose $\gamma \approx {B}^{-1/2}$ in \cref{alg:epoch} to make sure the label complexity is never greater than $B$ (with high probability).
The learned classifier enjoys Chow's excess error (with parameter $\gamma $) at most $\eps$; its standard excess error (with randomization over the abstention region) can be analyzed by relating the $\gamma \cdot \P_{x \sim \cD_\cX} \prn{ x \in \cX_{\gamma}}$ term in \cref{eq:prop_abstention} to the Massart/Tsybakov noise conditions, as discussed above.

\subsection{Abstention to avoid noise-seeking}
\label{sec:noise_seeking}

Active learning algorithms sometimes exhibit \emph{noise-seeking} behaviors,
i.e., oversampling in regions where $\eta(x)$ is close to the $\frac{1}{2}$ level.
Such noise-seeking behavior is known to be a fundamental barrier to achieve low label complexity (under standard excess error), e.g., see \citet{kaariainen2006active}.
We show in this section that abstention naturally helps avoiding noise-seeking behaviors and speeding up active learning.

To better illustrate how properly abstaining classifiers avoid noise-seeking behavior, we first introduce new noise conditions below, which strictly generalize the usual Massart/Tsybakov noises.
\begin{definition}[Noise-seeking Massart noise]
	\label{def:noise_seeking_Massart}
  A distribution $\cD_{\cX \cY}$ satisfies the noise-seeking Massart noise condition with parameters $0 \leq \zeta_0 < \tau_0 \leq 1 /2  $ if $\P_{x \sim \cD_\cX} \prn{\zeta_0 < \abs{ \eta(x) - 1 / 2} \leq \tau_0} = 0$.	
\end{definition}

\begin{definition}[Noise-seeking Tsybakov noise]
	\label{def:noise_seeking_Tsybakov}
  A distribution $\cD_{\cX \cY}$ satisfies the noise-seeking Tsybakov noise condition with parameters $0 \leq \zeta_0 < 1 /2  $, $\beta \geq 0$ and a universal constant $c>0$ if $\P_{x \sim \cD_\cX} \prn{ \zeta_0 < \abs{ \eta(x) - 1 / 2} \leq \tau} \leq  c\, \tau^\beta$ for any $\tau > \zeta_0$.	
\end{definition}

Compared to the standard Massart/Tsybakov noises, these newly introduced noise-seeking conditions allow arbitrary probability mass of data points whose conditional probability $\eta(x)$ is equal/close to $1/2$. 
As a result, they can trick standard active learning algorithms into exhibiting the noise-seeking bahaviors (and hence their names).
We also mention that the parameter $\zeta_0$ should be considered as an \emph{extremely small quantity} (e.g., $\zeta_0 \ll \eps$), with the extreme case corresponding to $\zeta_0 = 0$ (which still allow arbitrary probability for region $\crl{x \in \cX: \eta(x) = 1 /2 }$).

Ideally, any active learning algorithm should not be heavily affected by these noise conditions since it hardly matters (in terms of excess error) what label is predicted over region $\crl{x \in \cX: \abs{\eta(x) - 1 /2} \leq \zeta_0}$.
However, these seemingly benign noise-seeking conditions can cause troubles for any ``uncertainty-based'' active learner, i.e., any active learning algorithm that requests the label for any point that is uncertain (see \cref{def:proper_learner} in \cref{app:standard} for formal definition).
In particular, under limited budget, we derive the following result.

\begin{restatable}{proposition}{propBudget}
	\label{prop:budget}
Fix $\eps, \delta, \gamma > 0$.
    For any labeling budget $B \gtrsim \frac{1}{\gamma^2} \cdot \log^{2}\prn{\frac{1}{\eps \, \gamma}} \cdot \log \prn{\frac{1}{\eps \, \gamma \, \delta}}$, there exists a learning problem (with a set of linear regression functions) satisfying \cref{def:noise_seeking_Massart}/\cref{def:noise_seeking_Tsybakov} such that 
    (1) any ``uncertainty-based'' active learner suffers expected standard excess error $\Omega(B^{-1})$;
    yet (2) with probability at least $1-\delta$, \cref{alg:epoch} returns a classifier with standard excess error at most $\eps$.
\end{restatable}

The above result demonstrates the superiority of our \cref{alg:epoch} over any ``uncertainty-based'' active learner.
Moreover, we show that \cref{alg:epoch} achieves similar guarantees as in \cref{thm:standard_excess_error} under the strictly harder noise-seeking conditions.
Specifically, we have the following guarantees.

\begin{restatable}{theorem}{thmStandardExcessErrorNoise}
\label{thm:standard_excess_error_noise}
With an appropriate choice of the abstention parameter $\gamma$ in \cref{alg:epoch} and randomization over the abstention region, \cref{alg:epoch} learns a classifier $\check h$ with $\eps + \zeta_0$ standard excess error after querying $\wt \Theta(\tau_0^{-2})$ labels under \cref{def:noise_seeking_Massart} or querying $\wt \Theta \prn{ {\eps}^{ - 2 / (1 + \beta)} }$ labels under \cref{def:noise_seeking_Tsybakov}.
\end{restatable}

The special case of the noise-seeking condition with $\zeta_0 = 0$ is recently studied in \citep{kpotufe2021nuances}, where the authors conclude that no active learners can outperform the passive counterparts in the \emph{nonparametric} regime.
\cref{thm:standard_excess_error_noise} shows that, in the \emph{parametric} setting (with function approximation), \cref{alg:epoch} provably overcomes these noise-seeking conditions.

\section{Extensions}
\label{sec:extension}
We provide two adaptations of our main algorithm (\cref{alg:epoch}) that can 
(1) achieve constant label complexity for a general set of regression functions (\cref{sec:constant}); 
and (2) adapt to model misspecification (\cref{sec:misspecified}).
These two adaptations can also be efficiently implemented via regression oracle and enjoy similar guarantees stated in \cref{thm:epoch_efficient}. We defer computational analysis to \cref{app:constant} and \cref{app:mis}.

\subsection{Constant label complexity}
\label{sec:constant}

We start by considering a simple problem instance with $\cX = \crl{x}$, where active learning is reduced to mean estimation of $\eta(x)$.
Consider the Massart noise case where $\eta(x) \notin [\frac{1}{2} - \tau_0, \frac{1}{2} + \tau_0]$.
No matter how small the desired accuracy level $\epsilon>0$ is, the learner should not spend more than $O(\frac{\log(1/\delta)}{\tau_0^2})$ labels to correctly classify $x$ with probability at least $1-\delta$, which ensures $0$ excess error. 
In the general setting, but with Chow's excess error, a similar result follows: 
It takes at most $O(\frac{\log(1/\delta)}{\gamma^2})$ samples to verify if $\eta(x)$ is contained in $[\frac{1}{2}-\gamma, \frac{1}{2} + \gamma]$ or not. 
Taking the optimal action within $\crl{0, 1,\bot}$ (based on \cref{eq:chow_error}) then leads to $0$ Chow's excess error.
This reasoning shows that, at least in simple cases, one should be able to achieve \emph{constant} label complexity no matter how small $\epsilon$ is. One natural question to ask is as follows.

\begin{minipage}[c]{\linewidth}
\vspace{2 pt}
\centering
\emph{Is it possible to achieve constant label complexity in the general case of active learning?}
\vspace{2 pt}
\end{minipage}

We provide the first affirmative answer to the above question with a \emph{general} set of regression function $\cF$ (finite), and under \emph{general} action space $\cX$ and marginal distribution $\cD_{\cX}$.
The positive result is achieved by \cref{alg:eluder} (deferred to \cref{app:constant_alg}), which differs from \cref{alg:epoch} in two aspects: 
(1) we drop the epoch scheduling, and 
(2) apply a tighter elimination step derived from an optimal stopping theorem.
Another change comes from the analysis of the algorithm: Instead of analyzing with respect to the disagreement coefficient, we work with the \emph{eluder dimension} $\mfe \ldef \sup_{f^{\star} \in \cF}\mfe_{f^\star}(\cF,\gamma/2)$.\footnote{We formally define eluder dimension in \cref{app:star_eluder}. As examples, we have $\mfe = O(d \cdot \log \frac{1}{\gamma})$ for linear functions in $\R^d$, and $\mfe = O(C_{\textsf{link}} \cdot d \log \frac{1}{\gamma})$ for generalized linear functions (where $C_{\textsf{link}}$ is a quantity related to the link function).} 
To do that, we analyze active learning from the perspective of \emph{regret minimization with selective querying} \citep{dekel2012selective}, which allows us to incorporate techniques developed in the field of contextual bandits \citep{russo2013eluder, foster2020instance}.
We defer a detailed discussion to \cref{app:constant_regret} and provide the following guarantees. 
\begin{restatable}{theorem}{thmConstant}
\label{thm:constant}
With probability at least $1-2\delta$, \cref{alg:eluder} returns a classifier with expected Chow's excess error at most $\epsilon$ and label complexity $O\paren{ \frac{\mfe \cdot \log ({\abs*{\cF}}/{\delta})}{\gamma^2}}$, which is independent of $\frac{1}{\eps}$.
\end{restatable}

Based on discussion in \cref{sec:standard_excess_error}, we can immediately translate the above results into \emph{standard} excess error guarantees under the Massart noise (with $\gamma$ replaced by $\tau_0$).
We next discuss why existing algorithms/analyses do not guarantee constant label complexity, even in the linear case.
\begin{enumerate}
	\item \textbf{Epoch scheduling.} Many algorithms proceed in epochs and aim at \emph{halving} the excess error after each epoch \citep{balcan2007margin, zhang2014beyond, puchkin2021exponential}.
	One inevitably needs $\log \frac{1}{\eps}$ epochs to achieve $\epsilon$ excess error. 
	\item \textbf{Relating to disagreement coefficient.}
	The algorithm presented in \citet{krishnamurthy2019active} does not use epoch scheduling. However, their label complexity are analyzed with disagreement coefficient, which incurs a $\sum_{t=1}^{1/\eps} \frac{1}{t} = O(\log \frac{1}{\eps})$ term in the label complexity.
\end{enumerate}

\begin{remark}
    \cref{alg:eluder} also provides guarantees when $x$ is selected by an adaptive adversary (instead of i.i.d. sampled $x \sim \cD_\cX$). In that case, we simultaneously upper bound the regret and the label complexity (see \cref{thm:constant_adv} in \cref{app:constant_alg}). Our results can be viewed as a generalization of the results developed in the linear case \citep{dekel2012selective}.
\end{remark}

\subsection{Dealing with model misspecification}
\label{sec:misspecified}

Our main results are developed under realizability (\cref{asmp:predictable}), which assumes that there exists a $f^\star \in \cF$ such that $f^\star = \eta$. 
In this section, we relax that assumption and allow model misspecification.
We assume the learner is given a set of regression function $\cF: \cX \to [0, 1]$ that may only \emph{approximates} the conditional probability $\eta$. 
More specifically, we make the following assumption.

\begin{assumption}[Model misspecification]
\label{asmp:misspecified}
There exists a $\wb f \in \cF$ such that $\wb f$ approximate  $\eta$ up to $\kappa > 0$ accuracy, i.e., $\sup_{x \in \cX} \abs*{\bar f(x) - \eta(x)} \leq \kappa$.
\end{assumption}

We use a variation of \cref{alg:epoch} to adapt to model misspecification (\cref{alg:mis}, deferred to \cref{app:mis_alg}). 
Compared to \cref{alg:epoch}, the main change in \cref{alg:mis} is to apply a more conservative step in determining the active set $\cF_m$ at each epoch:
We maintain a larger active set of regression function to ensure that $\wb f$ is not eliminated throughout all epochs.
Our algorithm proceeds \emph{without} knowing the misspecification level $\kappa$. 
However, the excess error bound presented next holds under the condition that $\kappa \leq \eps$ (i.e., it requires that the misspecification is no larger than the desired accuracy). 
Abbreviate $\wb \theta \ldef \sup_{\iota > 0}\theta_{\wb f}^{\val}(\cF, \gamma /2, \iota)$, we achieve the following guarantees.

\begin{restatable}{theorem}{thmMis}
\label{thm:mis}
Suppose $\kappa\leq \eps$. With probability at least $1-2\delta$, \cref{alg:mis} returns a classifier with Chow's excess error 
$O \prn{ \eps \cdot  \wb \theta \cdot {\log\prn{\frac{ \pseud(\cF)}{\eps \, \gamma \, \delta}}}}$
and label complexity
$O\prn{ \frac{\wb \theta \, \pseud(\cF)}{\gamma^2} \cdot \log^2 \prn{\frac{\pseud(\cF)}{\eps \, \gamma}}  \cdot  {\log\prn{\frac{ \pseud(\cF)}{\eps \, \gamma \, \delta}}}}$.
\end{restatable}

We only provide guarantee when $\kappa \leq \eps$, since the learned classifier suffers from an additive $\kappa$ term in the excess error (see \cref{app:mis_partial} for more discussion).
On the other hand, the (inefficient) algorithm provided in \citet{puchkin2021exponential} works without any assumption on the approximation error.
An interesting future direction is to study the relation between computational efficiency and learning with \emph{general} approximation error.

\section*{Acknowledgements}

We thank the anonymous reviewers for their helpful comments. 
This work is partially supported by NSF grant 1934612 and AFOSR grant FA9550-18-1-0166.

\bibliography{refs.bib}
\clearpage

\appendix

\section{Disagreement coefficient, star number and eluder dimension}
\label{app:star_eluder}

We provide formal definitions/guarantees of value function disagreement coefficient, eluder dimension and star number in this section. These results are developed in \citet{foster2020instance, russo2013eluder}. 
Since our guarantees are developed in terms of these complexity measures, any future developments on these complexity measures (e.g., with respect to richer function classes) directly lead to broader applications of our algorithms.

We first state known upper bound on value function disagreement coefficient with respect to nice sets of regression functions.

\begin{proposition}[\citet{foster2020instance}]
	\label{prop:dis_coeff_bound}
	For any $f^{\star} \in \cF$ and $\gamma,\eps > 0$, let $\theta^{\val}_{f^{\star}}(\cF,\gamma,\eps)$ be the value function disagreement coefficient defined in \cref{def:dis_coeff}. Let $\phi: \cX \rightarrow \R^d$ be a fixed feature mapping and $\cW \subseteq \R^d$ be a fixed set. The following upper bounds hold true.
	\begin{itemize}
		\item Suppose $\cF \ldef \crl{x \mapsto \ang{\phi(x),w} : w \in \cW }$ is a set of linear functions. We then have $\sup_{f \in \cF, \gamma > 0 ,\eps > 0} \theta^{\val}_f(\cF, \gamma, \eps) \leq d$.
		\item Suppose $\cF \ldef \crl{ x \mapsto \sigma( \ang{ \phi(x), w}): w \in \cW}$ is a set of generalized linear functions with any fixed link function $\sigma: \R \rightarrow \R$ such that $0 < c_l < \sigma^{\prime} \leq c_u$.
		We then have 
		$\sup_{f \in \cF, \gamma > 0 ,\eps > 0} \theta^{\val}_f(\cF, \gamma, \eps) \leq {\frac{c_u}{c_l}}^2 \cdot d$.
	\end{itemize}
\end{proposition}

We next provide the formal definition of value function eluder dimension and star number \citep{foster2020instance, russo2013eluder}.

\begin{definition}[Value function eluder dimension]
\label{def:eluder}
For any $f^{\star} \in \cF$ and $\gamma > 0$, let $\check{\mathfrak{e}}_{f^{\star}}(\cF, \gamma)$ be the length of the longest sequence of data points $x^{1}, \dots, x^{m}$ such that for all $i$, there exists $f^{i} \in \cF$ such that 
\begin{align*}
    \abs{ f^{i}(x^{i}) - f^{\star}(x^{i}) } > \gamma, \quad \text{ and } \quad \sum_{j < i} \paren{ f^{i}(x^{j})  - f^{\star}(x^{j})}^2 \leq \gamma^2.
\end{align*}
The value function eluder dimension is defined as $\mathfrak{e}_{f^{\star}}(\cF, \gamma_0) \coloneqq \sup_{\gamma \geq \gamma_0} \check{\mathfrak{e}}_{f^{\star}}(\cF, \gamma)$. 
\end{definition}
\begin{definition}[Value function star number]
\label{def:star}
For any $f^{\star} \in \cF$ and $\gamma> 0$, let $\check{\mathfrak{s}}_{f^{\star}}(\cF, \gamma)$ be the length of the longest sequence of data points $x^{1}, \dots, x^{m}$ such that for all $i$, there exists $f^{i} \in \cF$ such that 
\begin{align*}
    \abs{ f^{i}(x^{i}) - f^{\star}(x^{i}) } > \gamma, \quad \text{ and } \quad \sum_{j \neq i} \paren{ f^{i}(x^{j})  - f^{\star}(x^{j})}^2 \leq \gamma^2.
\end{align*}
The value function eluder dimension is defined as $\mathfrak{s}_{f^{\star}}(\cF, \gamma_0) \coloneqq \sup_{\gamma \geq \gamma_0} \check{\mathfrak{s}}_{f^{\star}}(\cF, \gamma)$. 
\end{definition}

Since the second constrain in the definition of star number is more stringent than the counterpart in the definition of eluder dimension, one immediately have that $\mfs_{f^{\star}}(\cF, \gamma) \leq \mfe_{f^{\star}}(\cF, \gamma)$. 
We provide known upper bounds for eluder dimension next.

\begin{proposition}[\citet{russo2013eluder}]
	\label{prop:eluder_bound}
	Let $\phi: \cX \rightarrow \R^d$ be a fixed feature mapping and $\cW \subseteq \R^d$ be a fixed set. Suppose $\sup_{x \in \cX}\nrm{\phi(x)}_2 \leq 1$ and $\sup_{w \in \cW} \nrm{w}_2 \leq 1$. The following upper bounds hold true.
	\begin{itemize}
		\item Suppose $\cF \ldef \crl{x \mapsto \ang{\phi(x),w} : w \in \cW }$ is a set of linear functions. We then have $\sup_{f^{\star} \in \cF} \mfe_{f^{\star}}(\cF, \gamma) = O(d \log \frac{1}{\gamma})$.
		\item 
		Suppose $\cF \ldef \crl{ x \mapsto \sigma( \ang{ \phi(x), w}): w \in \cW}$ is a set of generalized linear functions with any fixed link function $\sigma: \R \rightarrow \R$ such that $0 < c_l < \sigma^{\prime} \leq c_u$.
        We then have 
	$\sup_{f^{\star} \in \cF} \mfe_{f^{\star}}(\cF, \gamma) = O \prn[\big]{ \prn[\big]{\frac{c_u}{c_l}}^2 d \log \prn[\big]{ \frac{c_u}{\gamma}} }$.
	\end{itemize}
\end{proposition}

The next result shows that the disagreement coefficient (with our \cref{def:dis_coeff}) can be always upper bounded by (squared) star number and eluder dimension.

\begin{proposition}[\citet{foster2020instance}]
    \label{prop:eluder_star_dis}
    Suppose $\cF$ is a uniform Glivenko-Cantelli class.
    For any $f^{\star}: \cX \rightarrow [0,1]$ and $\gamma, \eps >0$,
    we have $\theta^{\val}_{f^\star}(\cF, \gamma, \epsilon) \leq 4 \paren{\mfs_{f^\star}(\cF, \gamma)}^2$, and $\theta^{\val}_{f^\star}(\cF, \gamma, \epsilon) \leq 4 \, {\mfe_{f^\star}(\cF, \gamma)}$.
\end{proposition}

The requirement that $\cF$ is a uniform Glivenko-Cantelli class is rather weak: It is satisfied as long as  $\cF$ has finite Pseudo dimension \citep{anthony2002uniform}.

In our analysis, we sometimes work with sub probability measure (due to selective sampling). Our next result shows that defining the disagreement coefficient over all (sub) probability measures will not affect its value. 
More specifically, denote $\wt \theta^{\val}_{f^{\star}}(\cF, \gamma, \eps)$ be the disagreement coefficient defined in \cref{def:dis_coeff}, but with  $\sup$ taking over all probability and sub probability measures. We then have the following equivalence.

\begin{proposition}
	\label{prop:sub_measure}
	Fix any $\gamma_0, \eps_0 \geq 0$. We have  $\wt \theta^{\val}_{f^{\star}} (\cF, \gamma_0, \eps_0) = \theta^{\val}_{f^{\star}}  (\cF, \gamma_0, \eps_0)$.
\end{proposition}
\begin{proof}
We clearly have $\wt \theta^{\val}_{f^{\star}} (\cF, \gamma_0, \eps_0) \geq \theta^{\val}_{f^{\star}}  (\cF, \gamma_0, \eps_0)$ by additionally considering sub probability measures. We next show the opposite direction.

Fix any sub probability measure $\wt \cD_\cX$ that is non-zero (otherwise we have $\P_{x \sim \wt \cD_\cX} \prn{\cdot}  = 0$).
Suppose $\E_{x \sim \wt \cD_\cX} \brk{1} = \kappa < 1$. We can now consider its normalized probability measure  $\wb \cD_\cX$ such that  $\wb \cD_\cX(\omega) = \frac{\wt \cD_\cX(\omega)}{\kappa}$ (for any $\omega$ in the sigma algebra).
Now fix any $\gamma > \gamma_0$ and $\eps > \eps_0$. We have 
\begin{align*}
	& \frac{\gamma^2}{\eps^2} \cdot 
	\P_{\wt \cD_\cX} \prn*{ \exists f \in \cF: \abs{f(x) - f^{\star}(x)} > \gamma,
	\nrm*{ f - f^{\star}}_{\wt \cD_\cX}^2 \leq \eps^2} \\
	& = \frac{\gamma^2}{\eps^2 / \kappa} \cdot 
	\P_{\wb \cD_\cX} \prn*{ \exists f \in \cF: \abs{f(x) - f^{\star}(x)} > \gamma,
	\nrm*{ f - f^{\star}}_{\wb \cD_\cX}^2 \leq \eps^2 /\kappa} \\
	& = \frac{\gamma^2}{\wb \eps^2 } \cdot 
	\P_{\wb \cD_\cX} \prn*{ \exists f \in \cF: \abs{f(x) - f^{\star}(x)} > \gamma,
	\nrm*{ f - f^{\star}}_{\wb \cD_\cX}^2 \leq \wb \eps^2 } \\
	& \leq \theta^{\val}_{f^{\star}}(\cF, \gamma_0, \eps_0),
\end{align*}
where we denote $\wb \eps \ldef \frac{\eps}{\sqrt{\kappa}} > \eps$, and the last follows from the fact that $\wb \cD_\cX$ is a probability measure. We then have 
$\wt \theta^{\val}_{f^{\star}} (\cF, \gamma_0, \eps_0) \leq \theta^{\val}_{f^{\star}}  (\cF, \gamma_0, \eps_0)$, and thus the desired result.
\end{proof}

\section{Concentration results}
\label{app:concentration}

\begin{lemma}[Freedman's inequality, \citep{freedman1975tail, agarwal2014taming}]
    \label{lm:freedman}
    Let $(Z_t)_{t \leq T}$ be a real-valued martingale difference sequence adapted to a filtration $\mfF_t$, and let $\E_t \sq{\cdot} \ldef \E \sq{\cdot \mid \mfF_{t-1}}$. If $\abs{Z_t} \leq B$ almost surely, then for any $\eta \in (0,1/B)$ it holds with probability at least $1 - \delta$,
    \begin{align*}
        \sum_{t=1}^{T} Z_t \leq \eta \sum_{t=1}^{T} \E_{t} \sq{Z_t^2} + \frac{\log \delta^{-1}}{\eta}.
    \end{align*}
\end{lemma}

\begin{lemma}[\citep{foster2020instance}]
   \label{lm:martingale_two_sides} 
   Let $(X_t)_{t \leq T}$ be a sequence of random variables adapted to a filtration $\mfF_t$. If $0 \leq {X_t} \leq B$ almost surely, then with probability at least $1-\delta$,
   \begin{align*}
       \sum_{t =1 }^T X_t \leq \frac{3}{2} \sum_{t=1}^T \E_{t}\sq{X_t} + 4B \log(2 \delta^{-1}),
   \end{align*}
   and 
   \begin{align*}
       \sum_{t =1 }^T \E_{t} \sq{X_t} \leq 2 \sum_{t=1}^T X_t + 8B \log(2 \delta^{-1}).
   \end{align*}
\end{lemma}
\begin{proof}
  These two inequalities are obtained by applying \cref{lm:freedman} to $\prn{X_t - \E_t \sq{X_t}}_{t \leq T}$ and $\prn{\E_t \sq{X_t} - X_t}_{t \leq T}$, with $\eta = 1/2B$ and $\delta / 2$.	
  Note that $\E_t \sq{\prn{X_t - \E_t \sq{X_t}}^2} \leq \E_t \sq{X_t^2} \leq B \E_t\sq{X_t}$ if $0 \leq X_t \leq B$.
\end{proof}

We recall the definition of the Pseudo dimension of $\cF$.
\begin{definition}[Pseudo Dimension, \citet{pollard1984convergence, haussler1989decision, haussler1995sphere}]
\label{def:pseudo_d}
Consider a set of real-valued function $\cF: \cX \rightarrow \R$. The pseudo-dimension $\pdim(\cF)$ of $\cF$ is defined as the VC dimension of the set of threshold functions 
$\crl{(x,\zeta) \mapsto \ind(f(x) > \zeta) : f \in \cF}$.
\end{definition}
We next provide concentration results with respect to a general set of regression function $\cF$ with finite Pseudo dimension.
We define/recall some notations.
Fix any epoch $m \in [M]$ and any time step $t$ within epoch $m$.
For any $f \in \cF$, we denote $M_t(f) \ldef Q_t \prn{ \prn{f(x_t) - y_t}^2 - \prn{f^\star(x_t) - y_t}^2}$, 
and $\wh R_m(f) \ldef \sum_{t=1}^{\tau_{m-1}} Q_t \prn{f(x_t) - y_t}^2$.
Recall that we have $Q_t = g_m(x_t)$.
We define filtration $\mfF_t \ldef \sigma \prn{ \prn{x_1, y_1}, \ldots , \prn{x_{t}, y_{t}}}$,\footnote{$y_t$ is not observed (and thus not included in the filtration) when $Q_t = 0$. Note that $Q_t$ is measurable with respect to $\sigma( (\mfF_{t-1}, x_t) ) $.} 
and denote $\E_t \sq{\cdot } \ldef \E \sq{\cdot \mid \mfF_{t-1}}$.
\begin{lemma}[\citet{krishnamurthy2019active}]
    \label{lm:expected_sq_loss_pseudo}
    Suppose $\pseud(\cF) < \infty$.
   Fix any $\delta \in (0,1)$.  For any $\tau, \tau^\prime \in [T]$ such that $\tau < \tau^\prime$, with probability at least $1 - \delta $, we have 
   \begin{align*}
   	\sum_{t = \tau}^{\tau^\prime} M_t(f) \leq \sum_{t=\tau}^{\tau^\prime} \frac{3}{2} \E_t \brk{M_t(f)} + 
	C_\delta(\cF),
   \end{align*}
   and
   \begin{align*}
       \sum_{t = \tau}^{\tau^\prime} \E_t \sq{ M_t(f)}  \leq 2 \sum_{t = \tau}^{\tau^\prime} M_t(f) + C_\delta(\cF),
   \end{align*}
   where $C_\delta(\cF)  = C \cdot  \prn*{ \pdim(\cF) \cdot \log T +   \log \prn*{ \frac{\pdim(\cF)  \cdot  T}{\delta}} }  \leq C^{\prime} \cdot \prn*{\pseud(\cF) \cdot \log \prn*{\frac{T}{\delta}}}$, where $C, C^{\prime} >0$ are universal constants.
\end{lemma}

\section{Proofs of results in \cref{sec:epoch}}
\label{app:epoch}
We give the proof of \cref{thm:epoch} and \cref{thm:epoch_efficient}.
Supporting lemmas used in the proofs are deferred to \cref{app:epoch_lms}.

Fix any classifier $\wh h: \cX \rightarrow \crl{0, 1,\bot}$. For any $x\in\cX$, we introduce the notion
\begin{align}
& 	\exc_{\gamma}( \wh h;x) \ldef \nonumber\\
    &  \P_{y\mid x} \prn[\big]{y \neq \widehat h(x)} \cdot \ind \prn[\big]{ \widehat h(x) \neq \bot} + \prn[\big]{{1}/{2} - \gamma} \cdot \ind \prn[\big]{\widehat h(x) = \bot} - \P_{y\mid x} \prn[\big]{ y \neq h^\star(x) }\nonumber\\
    & = \ind \prn[\big]{ \widehat h(x) \neq \bot} \cdot \prn[\big]{\P_{y\mid x} \prn[\big]{y \neq \widehat h(x)} -  \P_{y\mid x} \prn[\big]{ y \neq h^\star(x) }} \nonumber \\
    & \quad + \ind \prn[\big]{ \widehat h(x) = \bot} \cdot \prn[\big]{ \prn[\big]{{1}/{2} - \gamma}  -  \P_{y\mid x} \prn[\big]{ y \neq h^\star(x) }} \label{eq:excess_x}
\end{align}
to represent the excess error of $\wh h$ at point $x\in \cX$. Excess error of classifier $\wh h$ can be then written as $\exc_{\gamma}(\wh h) \ldef \err_\gamma(\wh h) - \err(h^{\star}) = \E_{x \sim \cD_\cX} \brk{ \exc_{\gamma}(\wh h;x)}$.

\thmEpoch*
\begin{proof}
	We analyze under the good event $\cE$ defined in \cref{lm:expected_sq_loss_pseudo}, which holds with probability at least $1-{\delta}$. Note that all supporting lemmas stated in \cref{app:epoch_lms} hold true under this event.

We analyze the Chow's excess error of $\wh h_m$, which is measurable with respect to $\mfF_{\tau_{m-1}}$. 
For any $x \in \cX$, if $g_m(x) = 0$, 
\cref{lm:regret_no_query} implies that $\exc_{\gamma}(\wh h_m ;x) \leq 0$. 
If $g_m(x)= 1$, we know that $\wh h_m(x) \neq \bot$ and $\frac{1}{2} \in (\lcb(x;\cF_m),\ucb(x;\cF_m))$. 
Note that $\wh h_m(x) \neq h^{\star}(x)$ only if $\ind(f^{\star}(x) \geq 1/2) \neq \ind(\wh f_m(x) \geq 1/2)$.
Since $f^{\star}, \wh f_m \in \cF_m$ by \cref{lm:set_f}. 
The error incurred in this case can be upper bounded by $2 \abs{ f^{\star}(x)- 1 /2} \leq 2 w(x;\cF_m)$, which results in $\exc_{\gamma}(\wh h_m; x) \leq 2 w(x;\cF_m)$. 
Combining these two cases together, we have 
\begin{align*}
	\exc_{\gamma}( \wh h_m) \leq 2 \E_{x \sim \cD_\cX} \brk{ \ind(g_m(x) = 1) \cdot w(x;\cF_m)}.	
\end{align*}
Take $m=M$ and apply \cref{lm:per_round_regret_dis_coeff}, with notation $\rho_m \ldef 2 \beta_m + C_\delta$, leads to the following guarantee.
\begin{align*}
	\exc_{\gamma}( \wh h_M)
	& \leq  { \frac{ 8 \rho_M}{\tau_{M-1} \gamma} \cdot \theta^{\val}_{ f^{\star}}\prn*{\cF, \gamma/2, \sqrt{\rho_M/2 \tau_{M-1}}}}\\
	& =  O \prn*{ \frac{ \pseud(\cF) \cdot \log ( T / \delta)}{T \, \gamma} \cdot \theta^{\val}_{ f^{\star}}\prn*{\cF, \gamma/2, \sqrt{C_\delta/T}}},
\end{align*}
where we use the fact that $\frac{T}{2} \leq \tau_{M-1} \leq T$ and definitions of $\beta_m$ and $C_\delta$.
Simply considering $\theta \ldef \sup_{f^{\star} \in \cF, \iota > 0}\theta^{\val}_{f^{\star}} (\cF, \gamma / 2, \iota)$ as an upper bound of $\theta^{\val}_{f^{\star}} (\cF, \gamma / 2, \sqrt{C_\delta / T})$ 
and taking 
\begin{align*}
	T = O \prn*{ \frac{\theta \, \pseud(\cF) }{\eps \, \gamma} \cdot \log \prn*{ \frac{\theta \, \pseud(\cF) }{\eps \, \gamma \, \delta}} }	
\end{align*}
ensures that $\exc_{\gamma}(\wh h_M) \leq \eps$.

We now analyze the label complexity (note that the sampling process of \cref{alg:epoch} stops at time $t = \tau_{M-1}$).
Note that $\E \brk{\ind(Q_t = 1) \mid \mfF_{t-1}} = \E_{x\sim\cD_\cX} \brk{ \ind(g_m(x) = 1) }$ for any epoch $m \geq 2$ and time step $t$ within epoch $m$. 
Combining \cref{lm:martingale_two_sides} with \cref{lm:conf_width_dis_coeff} leads to
    \begin{align*}
        \sum_{t=1}^{\tau_{M-1}} \ind(Q_t = 1) & \leq \frac{3}{2} \sum_{t=1}^{\tau_{M-1}} \E \sq{\ind(Q_t = 1) \mid \mfF_{t-1}} + 4 \log \delta^{-1}\\
        & \leq 3 + \frac{3}{2}\sum_{m=2}^{M-1}\frac{(\tau_m - \tau_{m-1}) \cdot 4 \rho_m}{{\tau_{m-1}} \gamma^2} \cdot \theta^{\val}_{f^{\star}}\prn*{\cF, \gamma/2, \sqrt{\rho_m/2\tau_{m-1}}}  + 4 \log \delta^{-1} \\
        & \leq 3 + 6 \sum_{m=2}^{M-1}\frac{\rho_m}{ \gamma^2} \cdot \theta^{\val}_{f^{\star}}\prn*{\cF, \gamma/2, \sqrt{\rho_m/2\tau_{m-1}}}  + 4 \log \delta^{-1} \\
	& \leq 3 + 4 \log \delta^{-1} + \frac{18 \log T \cdot M \cdot C_\delta }{\gamma^2}
	\cdot \theta^{\val}_{f^{\star}}\prn*{\cF, \gamma/2, \sqrt{C_\delta/T }} \\ 
	& = O \prn*{ \frac{\theta \, \pseud(\cF) }{\gamma^2} \cdot \prn*{\log \prn*{ \frac{\theta \, \pseud(\cF) }{\eps \, \gamma}}}^{2} \cdot 
	\log \prn*{ \frac{\theta \,\pseud(\cF) }{\eps \, \gamma \, \delta}}} ,
    \end{align*}
    with probability at least $1-2\delta$ (due to an additional application of \cref{lm:martingale_two_sides}); where we plug the above choice of $T$ and upper bound other terms as before.
\end{proof}

\paragraph{A slightly different guarantee for \cref{alg:epoch}}
The stated \cref{alg:epoch} takes $\theta \ldef \sup_{f^{\star} \in \cF, \iota > 0}\theta^{\val}_{f^{\star}} (\cF, \gamma / 2, \iota)$ as an input (the value of $\theta$ can be upper bounded for many function class $\cF$, as discussed in \cref{app:star_eluder}).
However, we don't necessarily need to take $\theta$ as an input to the algorithm. Indeed, 
we can simply run a modified version of \cref{alg:epoch} with $T = \frac{\pseud(\cF)}{\eps \, \gamma}$.
Following similar analyses in proof of \cref{thm:epoch}, set $\iota \ldef \sqrt{C_\delta /T} \propto \sqrt{\gamma \eps}$, the modified version achieves excess error 
\begin{align*}
	\exc_{\gamma}(\wh h_M) = O \prn*{\eps \cdot \theta^{\val}_{f^{\star}}(\cF, \gamma /2, \iota) \cdot \log \prn*{\frac{\pseud(\cF)}{\eps \, \delta \, \gamma}}}
\end{align*}
with label complexity 
\begin{align*}
	 O \prn*{ \frac{\theta^{\val}_{f^{\star}}(\cF, \gamma /2, \iota) \cdot  \pseud(\cF) }{\gamma^2} \cdot \prn*{\log \prn*{ \frac{ \pseud(\cF) }{\eps \, \gamma}}}^{2} \cdot 
	\log \prn*{ \frac{\pseud(\cF) }{\eps \, \gamma \, \delta}}} .
\end{align*}

We now discuss the efficient implementation of \cref{alg:epoch} and its computational complexity. 
We first state some known results in computing the confidence intervals with respect to a set of regression functions $\cF$.

\begin{proposition}[\citet{krishnamurthy2017active, foster2018practical, foster2020instance}] \label{prop:CI_oracle}
Consider the setting studied in \cref{alg:epoch}. 
Fix any epoch $m \in [M]$ and denote $\cB_m \ldef \crl{ (x_t,Q_t, y_t)}_{t=1}^{\tau_{m-1}}$.
Fix any $\alpha > 0$.
For any data point $x \in \cX$, there exists algorithms $\AlgLcb$ and $\AlgUcb$ that certify
\begin{align*}
    & \lcb(x;\cF_m) - \alpha \leq \AlgLcb(x;\cB_m,\beta_m,\alpha) \leq \lcb(x;\cF_m) \quad \text{and}\\
    &\ucb(x;\cF_m) \leq \AlgUcb(x;\cB_m,\beta_m,\alpha) \leq \ucb(x;\cF_m) + \alpha.
\end{align*}
The algorithms take 
 $O(\frac{1}{\alpha^2} \log \frac{1}{\alpha})$ calls of the regression oracle for general $\cF$
 and take $O(\log \frac{1}{\alpha})$ calls of the regression oracle if $\cF$ is convex and closed under pointwise convergence.
\end{proposition}
\begin{proof}
See Algorithm 2 in \citet{krishnamurthy2017active} for the general case; and Algorithm 3 in \citet{foster2018practical} for the case when $\cF$ is convex and closed under pointwise convergence.
\end{proof}

We next discuss the computational efficiency of \cref{alg:epoch}. Recall that we redefine $\theta \ldef \sup_{f^{\star} \in \cF, \iota > 0} \theta^{\val}_{f^{\star}}(\cF,\gamma / 4, \iota)$ in the \cref{thm:epoch_efficient} to account to approximation error.
\thmEpochEfficient*
\begin{proof}
Fix any epoch $m \in [M]$.
Denote $\wb \alpha \ldef \frac{\gamma}{4M}$ and $\alpha_m \ldef \frac{(M-m) \gamma}{4M}$.
With any observed $x \in \cX$, we construct the approximated confidence intervals $\wh \lcb(x;\cF_m)$ and 
$\wh \ucb(x; \cF_m)$ as follows.
\begin{align*}
&	\wh \lcb(x;\cF_m) \ldef \AlgLcb(x;\cB_m,\beta_m,\wb \alpha) - \alpha_m \quad \text{and}	
    & \wh \ucb(x;\cF_m) \ldef\AlgUcb(x;\cB_m,\beta_m,\wb \alpha)+ \alpha_m. 
\end{align*}
For efficient implementation of \cref{alg:epoch}, we replace $\lcb(x;\cF_m)$ and $\ucb(x;\cF_m)$ with $\wh \lcb(x;\cF_m)$ and $\wh \ucb(x;\cF_m)$ in the construction of $\wh h_m$ and $g_m$.

Based on \cref{prop:CI_oracle}, we know that 
\begin{align*}
    & \lcb(x;\cF_m) - \alpha_m - \wb \alpha \leq 	\wh \lcb(x;\cF_m) \leq \lcb(x;\cF_m) - \alpha_m \quad \text{and}\\
    &\ucb(x;\cF_m) + \alpha_m \leq \wh \ucb(x;\cF_m) \leq \ucb(x;\cF_m) + \alpha_m + \wb \alpha .
\end{align*}
Since $\alpha_m + \wb \alpha  \leq \frac{\gamma}{4}$ for any $m \in [M]$, the guarantee in \cref{lm:query_implies_width} can be modified as $g_m(x)= 1 \implies w(x;\cF_m)\geq \frac{\gamma}{2}$. 

Fix any $m \geq 2$. Since $\cF_{m} \subseteq \cF_{m-1}$ by \cref{lm:set_f}, we have 
\begin{align*}
	& \wh \lcb(x;\cF_m) \geq \lcb(x;\cF_m) -  \alpha_m - \wb \alpha \geq \lcb(x;\cF_{m-1}) - \alpha_{m-1} \geq \wh\lcb (x;\cF_{m-1}) \quad \text{and} \\
	& \wh \ucb(x;\cF_m) \leq \ucb(x;\cF_m) +  \alpha_m + \wb \alpha \leq \ucb(x;\cF_{m-1}) + \alpha_{m-1} \leq \wh\ucb (x;\cF_{m-1}).
\end{align*}
These ensure $\ind(g_m(x) = 1) \leq \ind(g_{m-1}(x)=1)$. Thus, the guarantees stated in \cref{lm:conf_width_dis_coeff} and \cref{lm:per_round_regret_dis_coeff} still hold (with $\frac{\gamma}{2}$ replaced by $\frac{\gamma}{4}$ due to modification of \cref{lm:query_implies_width}).
The guarantee stated in \cref{lm:regret_no_query} also hold since  $\wh \lcb(x;\cF_m) \leq \lcb(x;\cF_m)$ and $\wh \ucb(x;\cF_m) \geq \ucb(x;\cF_m)$ by construction. As a result, the guarantees stated in \cref{thm:epoch} hold true with changes only in constant terms.

We now discuss the computational complexity of the efficient implementation. 
At the beginning of each epoch $m$. We use one oracle call to compute $\widehat f_m = \argmin_{f \in \cF} \sum_{t =1}^{ \tau_{m-1}} Q_t \paren{f(x_t) - y_t}^2 $. 
The main computational cost comes from computing $\wh \lcb$ and $\wh \ucb$ at each time step.
We take $\alpha = \wb \alpha \ldef \frac{\gamma}{4M}$ into \cref{prop:CI_oracle}, which leads to 
$O \prn{ \frac{(\log T)^2}{\gamma^2}\cdot \log \prn{ \frac{\log T}{\gamma}}}$ calls of the regression oracle for general $\cF$ and 
$O \prn{ \log \prn{ \frac{\log T}{\gamma}}}$ calls of the regression oracle for any convex $\cF$ that is closed under pointwise convergence. This also serves as the per-example inference time for $\wh h_{M}$. The total computational cost of \cref{alg:epoch} is then derived by multiplying the per-round cost by $T$ and plugging $T = \wt O( \frac{\theta \, \pseud(\cF)}{\eps \, \gamma})$ into the bound (for any parameter, we only keep $\poly$ factors in the total computational cost and keep $\poly$ or $\polylog$ dependence in the per-example computational cost).
\end{proof}

\subsection{Supporting lemmas}
\label{app:epoch_lms}
We use $\cE$ to denote the good event considered in \cref{lm:expected_sq_loss_pseudo}, and analyze under this event in this section. We abbreviate $C_\delta \ldef C_\delta(\cF)$ in the following analysis.

\begin{lemma}
\label{lm:set_f}
The followings hold true:
\begin{enumerate}
	\item $f^\star \in \cF_m$ for any $m \in [M]$.
	\item $\sum_{t=1}^{\tau_{m-1}} \E_t \brk{M_t(f)} \leq 2 \beta_m +  C_\delta$ for any $f \in \cF_m$. 
	\item $\cF_{m+1} \subseteq \cF_m$ for any $m \in [M-1]$.
\end{enumerate}
\end{lemma}
\begin{proof}
\begin{enumerate}
	\item 
	Fix any epoch $m \in [M]$ and time step $t$ within epoch $m$.
	Since $ \E [y_t] = f^\star(x_t)$, we have $\E_t \brk{ M_t(f)}  = \E \brk{ Q_t \prn{f(x) - f^\star(x) }^2 }=\E \brk{ g_m(x)\prn{f(x) - f^\star(x) }^2 } \geq 0$ for any $f \in \cF$.
By \cref{lm:expected_sq_loss_pseudo}, we then have $\wh R_m (f^\star) \leq \wh R_m(f) + C_\delta /2 \leq \wh R_m(f) + \beta_m$ for any $f \in \cF$.
The elimination rule in \cref{alg:eluder} then implies that $f^\star \in \cF_m$ for any $m \in [M]$.
\item  Fix any $f \in \cF_m$. With \cref{lm:expected_sq_loss_pseudo}, we have 
	\begin{align*}
		\sum_{t=1}^{\tau_{m-1}} \E_t [M_t(f)] & \leq 2 \sum_{t=1}^{\tau_{m-1}} M_t(f) + C_\delta \\
			& = 2 \wh R_{m }(f) - 2\wh R_{m}(f^{\star}) + C_\delta \\
			& \leq 2 \wh R_{m}(f) - 2\wh R_{m} ( \wh f_m) + C_\delta\\
			& \leq 2 \beta_m + C_\delta , 
	\end{align*}
	where the third line comes from the fact that $\wh f_m$ is the minimizer of $\wh R_{m} (\cdot) $; and the last line comes from the fact that $f \in \cF_m$.
	\item Fix any $f \in \cF_{m+1}$. We have 
	\begin{align*}
		\wh R_{m} (f) - \wh R_{m} (\wh f_m) & \leq   
		\wh R_{m} (f) - \wh R_{m} (f^{\star}) + \frac{C_\delta}{2}\\
		& = \wh R_{m+1}(f) - \wh R_{m+1}(f^{\star}) 
		- \sum_{t=\tau_{m-1}+1}^{\tau_{m}} M_t(f) + \frac{C_\delta}{2}\\
		& \leq \wh R_{m+1} (f) - \wh R_{m+1} (\wh f_{m+1}) 
		- \sum_{t=\tau_{m-1}+1}^{\tau_{m}} \E_t [M_t(f)] /2 + {C_\delta}\\
		& \leq \beta_{m+1} + C_\delta \\
		& = \beta_m,
	\end{align*}	
	where the first line comes from \cref{lm:expected_sq_loss_pseudo}; the third line comes from the fact that $\wh f_{m+1}$ is the minimizer with respect to $\wh R_{m+1}$ and \cref{lm:expected_sq_loss_pseudo}; the last line comes from the definition of $\beta_m$.
\end{enumerate}
\end{proof}
\begin{lemma}
\label{lm:query_implies_width}
For any $m \in [M]$, we have $g_m(x)= 1 \implies w(x;\cF_m) > \gamma$.
\end{lemma}
\begin{proof}
We only need to show that $\ucb(x;\cF_m) - \lcb(x;\cF_m) \leq \gamma \implies g_m(x) = 0$. Suppose otherwise $g_m(x) = 1$, which implies that both 
\begin{align}
\label{eq:query_condition}
\frac{1}{2} \in \prn*{\lcb(x;\cF_m), \ucb(x;\cF_m)}  \quad \text{ and } \quad {\brk*{\lcb(x;\cF_m), \ucb(x;\cF_m)} \nsubseteq \brk*{  \frac{1}{2}-\gamma, \frac{1}{2} +\gamma } } .
\end{align}
If $\frac{1}{2} \in (\lcb(x;\cF_m), \ucb(x;\cF_m))$ and $\ucb(x;\cF_m) - \lcb(x;\cF_m) \leq \gamma$, we must have $\lcb(x;\cF_m) \geq  \frac{1}{2}- \gamma$ and $\ucb(x;\cF_m) \leq \frac{1}{2} + \gamma$, which contradicts with \cref{eq:query_condition}.
\end{proof}

We introduce more notations.
Fix any $m \in [M]$.
We use $n_m \ldef \tau_m - \tau_{m-1}$ to denote the length of epoch $m$,
and use abbreviation $\rho_m \ldef 2 \beta_m + C_\delta$.
    Denote $\prn{\cX, \Sigma, \cD_\cX}$ as the (marginal) probability space,
    and denote $\wb \cX_m \ldef \crl{x \in \cX: g_m(x) = 1} \in \Sigma$ be the region where query \emph{is} requested within epoch $m$.
    Since we have $\cF_{m+1} \subseteq \cF_m$ by \cref{lm:set_f}, we clearly have $\wb \cX_{m+1} \subseteq \wb \cX_m$.
    We now define a sub probability measure $\wb \mu_m \ldef ({\cD_\cX})_{\mid \wb \cX_m}$ such that $\wb \mu_m(\omega) = \cD_{\cX}\prn{ \omega \cap \wb \cX_m}$ for any $\omega \in \Sigma$. 
    Fix any time step $t$ within epoch $m$ and any $\wb m \leq m$.
    Consider any measurable function $F$ (that is $\cD_\cX$ integrable), we have 
    \begin{align}
    	\E_{x \sim \cD_\cX} \brk*{ \ind(g_m(x) = 1) \cdot F(x)}
	& = \int_{x \in \wb \cX_{m}} F(x) \, d \cD_\cX(x) \nonumber \\ 
	& \leq \int_{x \in \wb \cX_{\wb m}} F(x) \, d \cD_\cX(x)\nonumber \\ 
	& = \int_{x \in \cX} F(x) \, d \wb \mu_{\wb m} (x) \nonumber \\ 
	& \rdef \E_{x \sim \wb \mu_{\wb m}} \brk*{ F(x)}, \label{eq:change_of_measure} 
    \end{align}
    where, by a slightly abuse of notations, we use $\E_{x \sim \mu} \sq{\cdot}$ to denote the integration with any sub probability measure $\mu$. 
    In particular, \cref{eq:change_of_measure} holds with equality when $\wb m = m$.
\begin{lemma}
    \label{lm:conf_width_dis_coeff}
 Fix any epoch $m \geq 2$. 
 We have
	\begin{align*}
    \E_{x \sim \cD_\cX} \sq{\ind (g_m(x) = 1)} 
	 \leq  \frac{4 \rho_m}{{\tau_{m-1}} \gamma^2} \cdot \theta^{\val}_{f^{\star}}\prn*{\cF, \gamma/2, \sqrt{\rho_m/2\tau_{m-1}}}.
	\end{align*}
\end{lemma}
\begin{proof}
We know that $\ind(g_m(x) = 1) = \ind (g_m(x)= 1) \cdot \ind(w(x;\cF_m) > \gamma )$ from \cref{lm:query_implies_width}. Thus, for any $\wb m \leq m$, we have 
    \begin{align}
	    \E_{x \sim \cD_\cX} \sq{\ind(g_m(x)= 1)} 
	    &  = \E_{x \sim \cD_\cX} \sq{\ind(g_m(x)= 1) \cdot \ind(w(x;\cF_m)> \gamma )}\nonumber \\ 
	    & \leq \E_{x \sim \wb \mu_{\wb m} } \sq{\ind(w(x;\cF_m)> \gamma  )}\nonumber \\
    & \leq \E_{x \sim \wb \mu_{\wb m}} \prn[\Big]{ \ind \prn[\big]{\exists f \in \cF_m, \abs*{f(x) - f^{\star}(x)} > \gamma/2}} , \label{eq:conf_width_dis_coeff_1}
    \end{align}
where the second line uses \cref{eq:change_of_measure} and the last line comes from the facts that
    $f^{\star} \in \cF_m$ and $w(x;\cF_m) > \gamma  \implies \exists f \in \cF_m, \abs{f(x) - f^{\star}(x)} > {\gamma}/ {2}$. 

     For any time step $t$, let  $m(t)$ denote the epoch where  $t$ belongs to.
From \cref{lm:set_f}, we know that, $\forall f \in \cF_m$,  
    \begin{align}
\rho_m &
\geq \sum_{t=1}^{\tau_{m -1}} \E_{t} \sq[\Big]{ Q_t \prn[\big]{f(x_t) - f^{\star}(x_t)}^2} \nonumber \\
       &  = \sum_{t=1}^{\tau_{m -1}} \E_{x \sim \cD_\cX} \sq[\Big]{\ind(g_{m(t)}(x)=1) \cdot \prn[\big]{f(x) - f^{\star}(x)}^2} \nonumber \\
	  & = \sum_{\wb m=1}^{m-1} n_{\wb m} \cdot \E_{x \sim \wb \mu_{\wb m}} \brk*{  \prn*{f(x) - f^{\star}(x)}^2}\nonumber \\
	& = \tau_{m-1} \E_{x \sim \wb \nu_m} \brk*{  \prn*{f(x) - f^{\star}(x)}^2}, \label{eq:conf_width_dis_coeff_2}
    \end{align}
    where we use $Q_t = g_{m(t)}(x_t) = \ind(g_{m(t)}(x) = 1)$ and \cref{eq:change_of_measure} on the second line, and define a new sub probability measure 
    $$\wb \nu_m \ldef \frac{1}{\tau_{m-1}} \sum_{\wb m =1}^{m-1} n_{\wb m} \cdot \wb \mu_{\wb m}$$ on the third line. 

    Plugging \cref{eq:conf_width_dis_coeff_2} into \cref{eq:conf_width_dis_coeff_1} leads to the bound 
    \begin{align*}
	& \E_{x \sim \cD_\cX} \sq{\ind(g_m(x)= 1)} \\
	& \leq \E_{x \sim \wb \nu_{m}} \sq[\bigg]{\ind \prn[\Big]{\exists f \in \cF, \abs[\big]{f(x) - f^{\star}(x)} > \gamma/2, \E_{x \sim \wb \nu_m} \sq[\Big]{\prn[\big]{f(x) - f^{\star}(x)}^2} \leq \frac{\rho_m}{\tau_{m-1}}}},
    \end{align*}
    where we use the definition of $\wb \nu_m$ again (note that \cref{eq:conf_width_dis_coeff_1} works with any $\wb m \leq m$).  
    Combining the above result with the discussion around \cref{prop:sub_measure} and \cref{def:dis_coeff}, we then have 
    \begin{align*}
	 \E_{x \sim \cD_\cX} \sq{\ind(g_m(x)= 1)}  
	 \leq  \frac{4 \rho_m}{{\tau_{m-1}} \, \gamma^2} \cdot \theta^{\val}_{f^{\star}}\prn*{\cF, \gamma/2, \sqrt{\rho_m/2\tau_{m-1}}}.
    \end{align*}
\end{proof}

\begin{lemma}
    \label{lm:per_round_regret_dis_coeff}
Fix any epoch $m\geq 2$. We have 
    \begin{align*}
    	\E_{x \sim \cD_\cX} \sq{\ind(g_m(x)= 1)\cdot w(x;\cF_m)} 
 \leq { \frac{4 \rho_m}{\tau_{m-1}\,  \gamma} \cdot \theta^{\val}_{f^{\star}}\prn*{\cF, \gamma/2, \sqrt{\rho_m/2\tau_{m-1}}}}.
    \end{align*}
\end{lemma}
\begin{proof}
    Similar to the proof of \cref{lm:conf_width_dis_coeff}, we have 
    \begin{align*}
	 \E_{x \sim \cD_\cX} \sq{\ind(g_m(x)= 1)\cdot w(x;\cF_m)} 
	& = \E_{x \sim \cD_\cX} \sq{\ind(g_m(x)=1) \cdot \ind(w(x;\cF_m)> \gamma )\cdot w(x;\cF_m)} \\
	& \leq \E_{x \sim \wb \mu_{\wb m}} \sq{\ind(w(x;\cF_m)> \gamma )\cdot w(x;\cF_m)}
    \end{align*}
    for any $\wb m \leq m$. 
With $\wb \nu_m = \frac{1}{\tau_{m-1}} \sum_{\wb m =1}^{m-1} n_{\wb m} \cdot \wb \mu_{\wb m}$, we then have 
\begin{align*}
	 & \E_{x \sim \cD_\cX} \sq{\ind(g_m(x)= 1)\cdot w(x;\cF_m)} \\
	& \leq \E_{x \sim \wb \nu_{m}} \sq{\ind(w(x;\cF_m)> \gamma )\cdot w(x;\cF_m)}\\
        & \leq \E_{x \sim \wb \nu_{m}} \sq*{\ind(\exists f \in \cF_m, \abs[\big]{f(x) - f^{\star}(x)} > \gamma/2)\cdot \prn*{\sup_{f , f^\prime \in \cF_m} \abs*{f(x) - f^\prime(x)}}} \\
        & \leq 2\E_{x \sim \wb \nu_{m}} \sq*{\ind(\exists f \in \cF_m, \abs[\big]{f(x) - f^{\star}(x)} > \gamma/2)\cdot \prn*{\sup_{f \in \cF_m} \abs{f(x) - f^{\star}(x)}}} \\
        & \leq 2 \int_{\gamma/2}^1 \E_{x \sim \wb \nu_{m}} \sq*{\ind \prn*{\sup_{f \in \cF_m} \abs[\big]{f(x) - f^{\star}(x)} \geq \omega}} \, d \, \omega \\
& \leq  2 \int_{\gamma/2}^1  \frac{1}{\omega^2} \, d \, \omega \cdot \prn*{ \frac{\rho_m}{\tau_{m-1}} \cdot \theta^{\val}_{f^{\star}}\prn*{\cF, \gamma/2, \sqrt{\rho_m/2\tau_{m-1}}}}\\
& \leq { \frac{4 \rho_m}{\tau_{m-1} \, \gamma} \cdot \theta^{\val}_{f^{\star}}\prn*{\cF, \gamma/2, \sqrt{\rho_m/2\tau_{m-1}}}},
\end{align*}
where we use similar steps as in the proof of \cref{lm:conf_width_dis_coeff}.
\end{proof}
\begin{lemma}
\label{lm:regret_no_query}
Fix any $m \in [M]$. We have $\exc_{\gamma}(\wh h_m ;x) \leq 0$ if $g_m(x) = 0$.
\end{lemma}
\begin{proof}
	Recall that
\begin{align*}
	\exc_{\gamma}( \wh h;x) & =  \nonumber
      \ind \prn[\big]{ \widehat h(x) \neq \bot} \cdot \prn[\big]{\P_{y \mid x} \prn[\big]{y \neq \widehat h(x)} -  \P_{y\mid x} \prn[\big]{ y \neq h^\star(x) }} \nonumber \\
    & \quad + \ind \prn[\big]{ \widehat h(x) = \bot} \cdot \prn[\big]{ \prn[\big]{{1}/{2} - \gamma}  -  \P_{y \mid x} \prn[\big]{ y \neq h^\star(x) }} .
\end{align*}
We now analyze the event $\curly*{g_m(x)= 0}$ in two cases. 

\textbf{Case 1: ${\widehat h_m(x) = \bot} $.} 

Since $\eta(x) = f^{\star}(x) \in [\lcb(x;\cF_m), \ucb(x;\cF_m)]$, we know that $\eta(x) \in \sq{ \frac{1}{2} - \gamma, \frac{1}{2} + \gamma }$ and thus $\P_{y \mid x} \prn[\big]{ y\neq h^\star(x) } \geq \frac{1}{2} - \gamma$. 
As a result, we have $\exc_{\gamma}(\wh h_m;x) \leq 0$.

\textbf{Case 2: ${\widehat h_m(x) \neq \bot}$ but ${\frac{1}{2} \notin (\lcb(x;\cF_m), \ucb(x;\cF_m))} $.} 

In this case, we know that $\widehat h_m (x) = h^\star(x)$ whenever $\eta(x) \in [\lcb(x;\cF_m), \ucb(x;\cF_m)]$. 
As a result, we have $\exc_{\gamma}(\wh h_m;x) \leq 0$ as well.
\end{proof}

\section{Proofs of results in \cref{sec:standard_excess_error}}
\label{app:standard}
\propProperAbs*
\begin{proof}
The proper abstention property of $\wh h$ returned by \cref{alg:epoch} is achieved via conservation: $\wh h$ will avoid abstention unless it is absolutely sure that abstention is the optimal choice.
The proper abstention property implies that $\P_{x \sim \cD_\cX} (\wh h(x) = \bot) \leq \P_{x \sim \cD_\cX} (x \in \cX_\gamma)$. The desired result follows by combining this inequality with \cref{eq:randomization}.
\end{proof}

\thmStandardExcessError*
\begin{proof}
	The results follow by taking the corresponding $\gamma$ in \cref{alg:epoch} and then apply \cref{prop:proper_abstention}.
	In the case with Massart noise, we have $\P_{x \sim \cD_\cX} (x \in \cX_{\gamma}) =0$ when $\gamma = \tau_0$; and the corresponding label complexity scales as $\wt O(\tau_0^{-2})$. 
	In the case with Tsybakov noise, we have $\gamma \cdot \P_{x \sim \cD_\cX} (x \in \cX_{\gamma}) = \frac{\eps}{2}$ when $\gamma = (\frac{\eps}{2c})^{1 /(1+\beta)}$. Applying \cref{alg:epoch} to achieve $\frac{\eps}{2}$ Chow's excess error thus leads to $\frac{\eps}{2} + \frac{\eps}{2} = \eps$ standard excess error. The corresponding label complexity scales as $\wt O(\eps^{-2 /(1+\beta)})$. 
\end{proof}

\thmStandardExcessErrorNoise*
\begin{proof}
For any abstention parameter $\gamma > 0$, we denote $\cX_{\zeta_0, \gamma} \ldef \crl{x \in \cX: \eta(x) \in \brk{\frac{1}{2} - \gamma , \frac{1}{2} + \gamma  }, \abs{\eta(x) - 1 /2 } > \zeta_0 }$ as the intersection of the region controlled by noise-seeking conditions and the (possible) abstention region. 
Let $\wh h$ be the classifier returned by \cref{alg:epoch} and $\check h$ be its randomized version (over the abstention region).
We denote $\cS \ldef \crl{x \in \cX: \wh h(x) = \bot}$ be the abstention region of $\wh h$. 
Since $\wh h$ abstains properly, we have $\cS \subseteq \crl{x \in \cX: \abs{\eta(x) - 1 /2} \leq \gamma} \rdef \cX_{\gamma} $. 
We write  $\cS_0 \ldef \cS \cap \cX_{\zeta_0, \gamma}$, $\cS_1 \ldef \cS \setminus \cS_0$ and $\cS_2 \ldef \cX \setminus \cS$.
For any $h: \cX \rightarrow \cY$, we use the notation $\exc(h;x) \ldef \prn{\P_{y\mid x} \prn[\big]{y \neq h(x)} -  \P_{y\mid x} \prn[\big]{ y \neq h^\star(x) }}$, and have $\exc(h) = \E_{x \sim \cD_\cX} \brk{\exc(h;x)}$.
We then have 
 \begin{align*}
	&\exc(\check h)\\
	& = \E_{x \sim \cD_\cX} \brk*{\exc(\check h; x) \cdot \ind \prn{x \in \cS_0}} 
	+ \E_{x \sim \cD_\cX} \brk*{\exc(\check h; x) \cdot \ind \prn{x \in \cS_1}}
	+ \E_{x \sim \cD_\cX} \brk*{\exc(\check h; x) \cdot \ind \prn{x \in \cS_2}}\\
	& \leq \gamma \cdot \E_{x \sim \cD_\cX} \brk{\ind(x \in \cS_0)}
	+ \zeta_0 \cdot \E_{x \sim \cD_\cX} \brk{\ind(x \in \cS_1)}
	+ \E_{x \sim \cD_\cX} \brk{\exc_{\gamma}(\wh h;x) \cdot \ind(x \in \cS_2)}\\
	& \leq \gamma \cdot \E_{x \sim \cD_\cX} \brk{\ind(x \in \cX_{\zeta_0,\gamma})} + \zeta_0 + \eps /2 ,
\end{align*}
where the bound on the third term comes from the same analysis that appears in the proof of \cref{thm:epoch} (with $\eps /2$ accuracy).
One can then tune $\gamma$ in ways discussed in the proof of \cref{thm:standard_excess_error} to bound the first term by $\eps /2$, i.e., 
$\gamma \cdot \E_{x \sim \cD_\cX} \brk{\ind(x \in \cX_{\zeta_0,\gamma})} \leq \eps /2$, with similar label complexity.
\end{proof}

\propBudget*

Before proving \cref{prop:budget}, we first construct a simple problem with linear regression function and give the formal definition of ``uncertainty-based'' active learner.

\begin{example}
   \label{ex:linear} 
   We consider the case where $\cX = [0,1]$ and $\cD_\cX = \unif(\cX)$. 
   We consider feature embedding $\phi:\cX \rightarrow \R^2$, i.e., $\phi(x) = \brk{ \phi_1(x), \phi_2(x)}^{\trn}$. 
   We take $\phi_1(x) \ldef 1$ for any $x \in \cX$, and define $\phi_2(x)$ as 
   \begin{align*}
       \phi_2(x) \ldef  \begin{cases}
	   0, & x \in \cX_{\hard}, \\
	   1, & x \in \cX_{\easy},
       \end{cases} 
   \end{align*}
   where $\cX_{\easy} \subseteq \cX$ is any subset such that $\cD_\cX(\cX_{\easy}) = p$, for some constant $p \in (0,1)$, and $\cX_{\hard} = \cX \setminus \cX_{\easy}$.
   We consider a set of linear regression function $\cF \ldef \crl{ f_\theta: f_\theta(x) = \ang{ \phi(x), \theta}, \nrm{\theta}_2 \leq 1}$.
   We set  $f^{\star} = f_{\theta^{\star}}$, where $\theta^{\star}= \brk{ \theta^{\star}_1, \theta_2^{\star}}^{\trn}$ is selected such that $\theta_1^{\star} = \frac{1}{2}$ and $\theta_2^{\star} = \unif(\crl{\pm \frac{1}{2}})$.
\end{example}

\begin{definition}
	\label{def:proper_learner}
	We say an algorithm is a ``uncertainty-based'' active learner if, for any $x \in \cX$, the learner
\begin{itemize}
	\item 	constructs an open confidence interval $(\lcb(x), \ucb(x))$ such that $\eta(x) \in (\lcb(x), \ucb(x))$;\footnote{By restricting to learners that construct an open confidence interval containing $\eta(x)$, we do not consider the corner cases when $\lcb(x) = \frac{1}{2}$ or $\ucb(x) = \frac{1}{2}$ and the confidence interval close.}
	\item  queries the label of $x \in \cX$ if $\frac{1}{2} \in (\lcb(x), \ucb(x))$.
\end{itemize}
\end{definition}
\begin{proof}
    With any given labeling budget $B$, we consider the problem instance described in \cref{ex:linear} with 
   $p = B^{-1}/2$.
   We can easily see that this problem instance satisfy \cref{def:noise_seeking_Massart} and \cref{def:noise_seeking_Tsybakov}.
    
    We first consider any ``uncertainty-based'' active learner. Let $Z$ denote the number of data points lie in $\cX_{\easy}$ among the first $B$ random draw of examples. 
    We see that $Z \sim \cB(B, B^{-1}/2)$ follows a binomial distribution with $B$ trials and $B^{-1}/2$ success rate.
    By Markov inequality, we have 
    \begin{align*}
	\P \prn*{Z \geq \frac{3}{2} \E[Z]} =  
	\P \prn*{Z \geq \frac{3}{4} } \leq \frac{2}{3}.
    \end{align*}
    That being said, with probability at least $1/3$, there will be $Z=0$ data point that randomly drawn from the easy region $\cX_{\easy}$. We denote that event as $\cE$. 
    Since $\eta(x) = f^{\star}(x) = \frac{1}{2}$ for any $x \in \cX_{\hard}$, any ``uncertainty-based'' active learner will query the label of any data point $x \in \cX_{\hard}$.
    As a result, under event $\cE$, the active learner will use up all the labeling budget in the first $B$ rounds and observe zero label for any data point $x \in \cX_{\easy}$.
    Since the easy region $\cX_{\easy}$ has measure $B^{-1} / 2$ and $\theta^{\star}_2 = \unif( \crl{ \pm \frac{1}{2}})$, any classification rule over the easy region would results in expected excess error lower bounded by
     $B^{-1}/ 4$.
     To summarize, with probability at least $\frac{1}{3}$, any ``uncertainty-based'' active learner without abstention suffers expected excess error $\Omega(B^{-1})$.

    We now consider the classifier returned by \cref{alg:epoch}.
    For the linear function considered in \cref{ex:linear}, we have $\pdim(\cF) \leq 2$ \citep{haussler1989decision} and $\theta^{\val}_{f^\star}(\cF, \gamma / 2, \eps) \leq 2$ for any $\eps \geq 0$ (see \cref{app:star_eluder}).
    Thus, by setting $T = O \prn{ \frac{1}{\eps \, \gamma} \cdot \log \prn{ \frac{1}{\eps \, \gamma \, \delta}}}$, with probability at least $1 - \delta$, \cref{alg:epoch} return a classifier $\wh h$ with Chow's excess error at most $\eps$ and label complexity $O \prn{ \frac{1}{\gamma^2} \cdot \log^2 \prn{ \frac{1}{\eps \, \gamma}} \cdot \log \prn{ \frac{1}{\eps \, \gamma \, \delta}}} = \poly \prn{ \frac{1}{\gamma}, \log \prn{ \frac{1}{\eps\, \gamma \, \delta}}} $. 
    Since $\wh h$ enjoys proper abstention, it never abstains for $x \in \cX_{\easy}$. 
    Note that we have $\eta(x) = \frac{1}{2}$ for any $x \in \cX_{\hard}$.
    By randomizing the decision of $\wh h$ over the abstention region, we obtain a classifier with standard excess error at most $\eps$. 
\end{proof}
\section{Omitted details for \cref{sec:constant}}
\label{app:constant}
We introduce a new perspective for designing and analyzing active learning algorithms in \cref{app:constant_regret}.
We present our algorithm and its theoretical guarantees in \cref{app:constant_alg}, and defer  
supporting lemmas to \cref{app:constant_lms}.

\subsection{The perspective: Regret minimization with selective sampling}
\label{app:constant_regret}
We view active learning as a decision making problem: at each round, the learner selects an action, suffers a loss (that may not be observable), and decides to query the label or not. 
At a high level, the learner aims at \emph{simultaneously} minimizing the regret and the number of queries. The leaner returns a (randomized) classifier/decision rule at the end of the learning process. 

The perspective is inspired by the seminal results derived in \citet{dekel2012selective}, where the authors study active learning with linear functions and focus on developing standard excess error guarantees.
With this regret minimization perspective, we can also take advantage of fruitful results developed in the field of contextual bandits \citep{russo2013eluder, foster2020instance}.

\paragraph{Decision making for regret minimization} To formulate the regret minimization problem, we consider the action set $\cA = \crl{0, 1, \bot}$, where the action $1$ (resp. $0$) represents labeling any data point $x\in \cX$ as $1$ (resp. $0$); and the action $\bot$ represents abstention. 
At each round $t \in [T]$, the learner observes a data point $x_t \in \cX$ (which can be chosen by an adaptive adversary), takes an action $a_t \in \cA$, and then suffers a loss, which is defined as  
\begin{align*}
    \ell_t(a_t) = \ind(y_t \neq a_t, a_t \neq \bot) + \prn*{\frac{1}{2} - \gamma} \cdot \ind (a_t = \bot).
\end{align*}
We use $a^{\star}_t \ldef \ind(f^{\star}(x_t) \geq 1/2) = \ind( \eta(x_t)\geq 1/2)$ to denote the action taken by the Bayes optimal classifier $h^{\star}\in \cH$. 
Denote filtration $\mfF_{t} \ldef \sigma(\paren*{x_i, y_i}_{i=1}^{t})$. We define the (conditional) expected regret at time step $t \in [T]$ as 
\begin{align*}
    {\reg}_t \ldef \E \sq{\ell_t(a_t) - \ell_t(a^\star_t)  \mid  \mfF_{t-1}}.
\end{align*}
The (conditional) expected cumulative regret across $T$ rounds is defined as
\begin{align*}
	{\reg}(T) \ldef \sum_{t=1}^T {\reg}_t,
\end{align*}
which is the target that the learner aims at minimizing.

\paragraph{Selective querying for label efficiency} Besides choosing an action $a_t \in \cA$ at each time step, our algorithm also determines \emph{whether or not} to query the label $y_t$ with respect to $x_t$. 
Note that such selective querying protocol makes our problem different from contextual bandits \citep{russo2013eluder, foster2020instance}:
The loss $\ell_t(a_t)$ of an chosen $a_t$ may not be even observed.

We use $Q_t$ to indicate the query status at round $t$, i.e., 
\begin{align*}
    Q_t = \ind \paren*{\text{label $y_t$ of $x_t$ is queried}}.
\end{align*}
The learner also aims at minimizing the total number of queries across $T$ rounds, i.e., $\sum_{t=1}^{T} Q_t$.

\paragraph{Connection to active learning} 
We consider the following learner for the above mentioned decision making problem with $(x, y) \sim \cD_{\cX \cY}$. 
At each round, the learner constructs a classifier $\wh h_t: \cX \rightarrow \crl{0,1,\bot}$ and
a query function $g_t:\cX \rightarrow \crl{0,1}$; the learner then takes action $a_t = \wh h_t(x_t)$ and decides the query status as $Q_t = g_t(x_t)$.

Conditioned on $\mfF_{t-1}$, taking expectation over $\ell_t(a_t)$ leads to the following equivalence:
\begin{align*}
   \E \brk*{\ell_t(a_t) \mid \mfF_{t-1}} & = \E \brk*{ \ind(y_t \neq a_t, a_t \neq \bot) + \prn*{\frac{1}{2} - \gamma} \cdot \ind (a_t = \bot) \mid \mfF_{t-1}}\\
   & =  \E \brk*{ \ind \prn[\big]{y_t \neq \widehat h(x_t), \widehat h(x_t) \neq \bot} + \prn*{\frac{1}{2} - \gamma} \cdot \ind \prn[\big]{\widehat h(x_t) = \bot}  \mid \mfF_{t-1}} \\
   & =  \P_{(x,y) \sim \cD_{\cX \cY}} \prn[\big]{y \neq \wh h(x), \wh h(x) \neq \bot} + \prn*{\frac{1}{2} - \gamma} \cdot \P \paren{\widehat h(x) = \bot}\\
   & = \err_{\gamma}\prn{\wh h_t}.
\end{align*}
This shows that the (conditional) expected instantaneous loss precisely captures the Chow's error of classifier $\widehat h_t$. Similarly, we have 
\begin{align*}
   \E \sq*{\ell_t(a_t^\star) \mid \mfF_{t-1} } = \P_{(x,y) \sim \cD_{\cX \cY}} \paren*{ \ind \paren{ y \neq \ind(\eta(x) \geq 1/2) }} = \err(h^\star).
\end{align*}
Combining the above two results, we notice that the (conditional) expected instantaneous \emph{regret} exactly captures the Chow's excess error of classifier $\widehat h_t$, i.e., 
\begin{align*}
     \reg_t  = \err_{\gamma}(\widehat h_t) - \err(h^\star).
\end{align*}
Let $\wh h \sim \unif( \crl{\wh h_t}_{t=1}^{T})$ be a classifier randomly selected from all the constructed classifiers. Taking expectation with respect to this random selection procedure, we then have 
\begin{align}
\E_{\wh h \sim \unif( \crl{\wh h_t}_{t=1}^{T})} \brk{ \err_{\gamma}(\widehat h) - \err(h^\star) } = \sum_{t=1}^T \paren{ \err_{\gamma}(\widehat h_t) - \err(h^\star) } / T = \reg(T)/T . 
\label{eq:expect_chow}
\end{align}
If we manage to guarantee that the cumulative regret is sublinear in $T$ and the total number of queries is logarithmic in $T$, we would achieve the goal of active learning with exponential savings in label complexity.

For analysis purpose, we also consider another classifier $\wh h_t^\star$, which is defined as 
\begin{align*}
  \wh h_t^\star(x) \ldef \begin{cases}
    \bot, & \text{if } \wh h_t(x) = \bot;\\
    h^\star(x), & \text{o.w.}
  \end{cases}
\end{align*}
That is, $\wh h_t^\star$ abstains whenever $\wh h_t$ abstains, and follows the Bayes optimal classifier otherwise. We use $\wh a_t^\star = \wh h_t(x_t)$ to denote the action of $\wh h_t^\star$ at round $t$ and have $\E \brk{\ell_t(a_t) \mid \mfF_{t-1}} = \err_\gamma(\wh h_t^\star)$.

\subsection{Algorithm and main results}
\label{app:constant_alg}

We present our algorithm that achieves constant label complexity in \cref{alg:eluder}. 
Compared to \cref{alg:epoch}, \cref{alg:eluder} drops the epoch scheduling, uses a sharper elimination rule for the active set (note that $\beta$ doesn't depend on $T$, thanks to the optimal stopping theorem in \cref{lm:opt_stopping}), and is analyzed with respect to eluder dimension (\cref{def:eluder}) instead of disagreement coefficient.
As a result, we shave all three sources of $\log \frac{1}{\eps}$, and achieve constant label complexity for general $\cF$ (as long as it's finite and has finite eluder dimension). 
We abbreviate $\mfe \ldef \sup_{f^{\star} \in \cF}\mfe_{f^{\star}}(\cF, \gamma /2)$.

\begin{algorithm}[H]
	\caption{Efficient Active Learning with Abstention (Constant Label Complexity)}
	\label{alg:eluder} 
	\renewcommand{\algorithmicrequire}{\textbf{Input:}}
	\renewcommand{\algorithmicensure}{\textbf{Output:}}
	\newcommand{\algorithmicbreak}{\textbf{break}}
    \newcommand{\BREAK}{\STATE \algorithmicbreak}
	\begin{algorithmic}[1]
		\REQUIRE Time horizon $T \in \N$, abstention parameter $\gamma \in (0, 1/2)$ and confidence level $\delta \in (0, 1)$.
		\STATE Initialize $\wh \cH \ldef \emptyset$. Set $T \ldef O \prn{\frac{\mfe}{\eps \, \gamma} \cdot \log \prn{\frac{\abs{\cF}}{\delta}}}$ and $\beta\ldef{2}	\log\prn[\big]{\frac{2 \abs*{\cF}}{\delta}}$. 
		\FOR{$t = 1, 2, \dots, T$}
		\STATE Get $\widehat f_t \ldef \argmin_{f \in \cF} \sum_{i < t} Q_i \paren{f(x_i) - y_i}^2 $.\\
		\hfill \algcommentlight{We use $Q_t \in \crl{0,1}$ to indicate whether the label of  $x_t$ is queried.}
		\STATE (Implicitly) Construct active set of regression function $\cF_t \subseteq \cF$ as 
		\begin{align*}
		    \cF_t \ldef \crl*{ f \in \cF:  \sum_{i = 1}^{t-1} Q_i \prn*{f(x_i) - y_i}^2 \leq \sum_{i = 1}^{t-1} Q_i \paren{\widehat f_t(x_i) - y_i}^2 + \beta}. 
		\end{align*}
		\STATE Construct classifier $\wh h_t:\cX \rightarrow \crl{0, 1,\bot}$ as 
		\begin{align*}
			\wh h_t (x) \ldef 
			\begin{cases}
				\bot, & \text{ if } \brk { \lcb(x;\cF_t), \ucb(x;\cF_t)} \subseteq 
				\brk*{ \frac{1}{2} - \gamma, \frac{1}{2} + \gamma}; \\
        \ind(\wh f_t(x) \geq \frac{1}{2} ) , & \text{o.w.}
			\end{cases}
		\end{align*}
		Update $\wh \cH = \wh \cH \cup \crl{\wh h_t}$.	
		Construct query function $g_m:\cX \rightarrow \crl{0,1}$ as
		\begin{align*}
		g_t(x)\ldef \ind \prn*{ \frac{1}{2} \in \prn{\lcb(x;\cF_t), \ucb(x;\cF_t)} } \cdot 	
		\ind \prn{\wh h_t(x) \neq \bot} .
		\end{align*}
		\STATE Observe $x_t \sim \cD_{\cX}$. Take action $a_t \ldef \wh h_t (x_t)$. Set $Q_t \ldef g_t(x_t)$. 
		\IF{$Q_t = 1$}
		\STATE Query the label $y_t$ of $x_t$.
		\ENDIF
		\ENDFOR
		\STATE \textbf{Return} $\wh h \ldef \unif(\wh \cH)$.
	\end{algorithmic}
\end{algorithm}

Before proving \cref{thm:constant}. We define some notations that are specialized to \cref{app:constant}.

We define filtrations $\mfF_{t-1} \ldef  \sigma ( x_1,y_1,\ldots, x_{t-1} ,y_{t-1})$ and $\wb \mfF_{t-1} \ldef  \sigma ( x_1,y_1,\ldots, x_{t-1}, y_{t-1}, x_t )$.
Note that we additionally include the data point $x_t$ in the filtration $\wb \mfF_{t-1}$ at time step $t-1$.
We denote $\E_t [\cdot] \ldef \E [\cdot \mid \wb \mfF_{t-1}]$. 
For any $t \in [T]$, we denote $M_t(f) \ldef Q_t \prn{ \prn{f(x_t) - y_t}^2 - \prn{f^\star(x_t) - y_t}^2}$. 
We have $\sum_{i=1}^{\tau} \E_t[M_t(f)] = \sum_{t=1}^{\tau} Q_t \prn{f(x_t) - f^\star(x_t)}^2$.
For any given data point $x_t \in \cX$, we use abbreviations
\begin{align*}
    \ucb_t \ldef \ucb(x_t; \cF_t) = \sup_{f \in \cF_t} f(x_t) \quad  \text{ and }  \quad \lcb_t \ldef \lcb(x_t; \cF_t) = \inf_{f \in \cF_t} f(x_t)
\end{align*}
to denote the upper and lower confidence bounds of $\eta(x_t) = f^\star(x_t)$. We also denote 
\begin{align*}
    w_t \ldef \ucb_t - \lcb_t = \sup_{f, f^\prime \in \cF_t} \abs*{f(x_t) - f^\prime(x_t)}
\end{align*}
as the width of confidence interval.

\thmConstant*

\begin{proof}
	We first analyze the label complexity of \cref{alg:eluder}.
	Note that \cref{alg:eluder} constructs $\wh h_t$ and $g_t$ in forms similar to the ones constructed in \cref{alg:epoch}, and \cref{lm:query_implies_width} holds for \cref{alg:eluder} as well.
	Based on \cref{lm:query_implies_width}, we have $Q_t = g_t(x_t) = 1 \implies w_t > \gamma$. 
	Thus, taking $\zeta = \gamma$ in \cref{lm:conf_width_eluder} leads to 
	$$
	\sum_{t=1}^{T} \ind (Q_t =1) < \frac{17 \log (2 \abs{\cF}/ \delta)}{2 \gamma^2} \cdot \mfe_{f^{\star}}(\cF, \gamma / 2),
	$$ 
	\emph{with probability one}. 
	The label complexity of \cref{alg:eluder} is then upper bounded by a constant as long as $\mfe_{f^{\star}}(\cF, \gamma / 2)$ is upper bounded by a constant (which has no dependence on $T$ or $\frac{1}{\eps}$).

We next analyze the excess error of $\wh h$. We consider the good event $\cE$ defined in \cref{lm:set_f_eluder}, which holds true with probability at least $1-\delta$.
Under event $\cE$, \cref{lm:regret_eluder_cond_wh_a} shows that 
\begin{align*}
	\sum_{t=1}^{T} \E \brk{\ell_t(a_t) - \ell_t(\wh a_t^\star) \mid \wb \mfF_{t-1} } \leq  
	\frac{17 \sqrt{2} \beta}{\gamma} \cdot \mfe_{f^{\star}}(\cF, \gamma / 2). 
\end{align*}
Since 
\begin{align*}
	\E \brk[\Big] {\E \brk{\ell_t(a_t) - \ell_t(\wh a_t^\star) \mid \wb \mfF_{t-1} } \mid \mfF_{t-1}}
	= \E \brk{\ell_t(a_t) - \ell_t(\wh a_t^\star) \mid \mfF_{t-1} },
\end{align*}
and 
  $ 0 \leq {\E \brk{\ell_t(a_t) - \ell_t(\wh a_t^\star) \mid \wb \mfF_{t-1} }}  \leq 1$ by \cref{lm:wh_a_ineq}, 
applying \cref{lm:martingale_two_sides} with respect to $\E \brk{\ell_t(a_t) - \ell_t(\wh a_t^\star) \mid \wb \mfF_{t-1} }$ leads to 
\begin{align*}
  \sum_{t=1}^{T} \E \brk{\ell_t(a_t) - \ell_t(\wh a_t^\star) \mid \mfF_{t-1} } \leq  
	\frac{34 \sqrt{2} \beta}{\gamma} \cdot \mfe_{f^{\star}}(\cF, \gamma / 2) + 8 \log(2 \delta^{-1}).
\end{align*}
  From \cref{lm:wh_a_ineq}, we know that 
  \begin{align*}
    \E \brk{\ell_t(\wh a^\star_t) - \ell_t(a_t^\star) \mid \mfF_{t-1} } = 
    \E \brk[\Big]{\E \brk{\ell_t(\wh a^\star_t) - \ell_t(a_t^\star) \mid \wb \mfF_{t-1} } \mid \mfF_{t-1}} \leq 0. 
  \end{align*}
  We then have 
\begin{align*}
  \reg(T) &=	\sum_{t=1}^{T} \E \brk{\ell_t(a_t) - \ell_t(a_t^\star) \mid \mfF_{t-1} }\\
  &=	\sum_{t=1}^{T} \E \brk{\ell_t(a_t) - \ell_t(\wh a_t^\star) \mid \mfF_{t-1} } + \sum_{t=1}^{T} \E \brk{\ell_t(\wh a^\star_t) - \ell_t(a_t^\star) \mid \mfF_{t-1} }\\
  &\leq  
	\frac{34 \sqrt{2} \beta}{\gamma} \cdot \mfe_{f^{\star}}(\cF, \gamma / 2) + 8 \log(2 \delta^{-1}),
\end{align*}
with probability at least $1-2\delta$ (due to the additional application of \cref{lm:martingale_two_sides}).
Since $\wh h \sim \unif(\wh \cH)$, based on \cref{eq:expect_chow}, we thus know that 
\begin{align*}
\E_{\wh h \sim \unif(\wh \cH)} \brk{ \err_\gamma(\wh h) - \err(h^{\star})} & = \sum_{t=1}^{T}\prn[\big]{\err_\gamma(\wh h_t) - \err( h^{\star})}/ T \\
	&\leq \prn*{\frac{34 \sqrt{2} \beta}{ \gamma} \cdot \mfe_{f^{\star}}(\cF, \gamma / 2) + 8 \log \prn*{2 \delta^{-1}}}/T 
\end{align*}
With $T \ldef  O \prn{\frac{\mfe}{\eps \, \gamma} \cdot \log \prn{ \frac{\abs{\cF}}{\delta}}}$, we can control the expected Chow's excess error to be at most $\eps$.
\end{proof}

\begin{theorem}
	\label{thm:constant_adv}
	Consider the setting where the data points $\crl{x_t}_{t=1}^{T}$ are chosen by an adaptive adversary with $y_t \sim \cD_{\cY \mid x_t}$. With probability at least $1-\delta$, \cref{alg:eluder} simultaneously guarantees 
\begin{align*}
	\sum_{t=1}^{T} \E \brk{\ell_t(a_t) - \ell_t(a_t^\star) \mid \wb \mfF_{t-1} } \leq  
	\frac{34 \sqrt{2} \beta}{\gamma} \cdot \mfe_{f^{\star}}(\cF, \gamma / 2),
\end{align*}
and 
	$$
	\sum_{t=1}^{T} \ind (Q_t =1) < \frac{17 \log (2 \abs{\cF}/ \delta)}{2 \gamma^2} \cdot \mfe_{f^{\star}}(\cF, \gamma / 2).
	$$ 
\end{theorem}
\begin{proof}
  The label complexity follows the same analysis as in the proof of \cref{thm:constant}.

  To analyze the regret, we consider the good event $\cE$ defined in \cref{lm:set_f_eluder}, which holds true with probability at least $1-\delta$.
Under event $\cE$, \cref{lm:regret_eluder_cond} shows that 
\begin{align*}
	\sum_{t=1}^{T} \E \brk{\ell_t(a_t) - \ell_t(a_t^\star) \mid \wb \mfF_{t-1} } \leq  
	\frac{17 \sqrt{2} \beta}{\gamma} \cdot \mfe_{f^{\star}}(\cF, \gamma / 2). 
\end{align*}
\end{proof}

We redefine $\mfe \ldef \sup_{f^{\star} \in \cF}\mfe_{f^{\star}}(\cF, \gamma /4)$ in the following \cref{thm:eluder_efficient} to account for the induced approximation error in efficient implementation.
\begin{restatable}{theorem}{thmEluderEfficient}
	\label{thm:eluder_efficient}
	\cref{alg:eluder} can be efficiently implemented via the regression oracle and enjoys the same theoretical guarantees stated in \cref{thm:constant} or \cref{thm:constant_adv}.
	The number of oracle calls needed is $ O(\frac{\mfe}{\eps \, \gamma^{3}} \cdot \log \prn{\frac{\abs{\cF}}{\delta}}\cdot \log \prn{\frac{1}{\gamma}})$ for a general set of regression functions $\cF$, and $ O(\frac{\mfe}{\eps \, \gamma} \cdot \log \prn{\frac{\abs{\cF}}{\delta}} \cdot \log \prn{\frac{1}{\gamma}})$ when $\cF$ is convex and closed under pointwise convergence.
	The per-example inference time of the learned $\wh h_{M}$ is $O ( \frac{1}{\gamma^2} \log \frac{1}{\gamma})$ for general $\cF$, and $O ( \log \frac{1}{\gamma}) $ when $\cF$ is convex and closed under pointwise convergence.
\end{restatable}
\begin{proof}
Denote $\cB_t \ldef \crl{ (x_i,Q_i, y_i)}_{i=1}^{\tau_{t-1}}$ 
At any time step $t \in [T]$ of \cref{alg:eluder}, we construct classifier $\wh h_t$ and query function $g_t$ with approximated confidence bounds, i.e.,
\begin{align*}
&	\wh \lcb(x;\cF_t) \ldef \AlgLcb(x;\cB_t,\beta_t,\alpha)  \quad \text{and}	
    & \wh \ucb(x;\cF_t) \ldef\AlgUcb(x;\cB_t,\beta_t,\alpha), 
\end{align*}
where $\AlgLcb$ and $\AlgUcb$ are subroutines discussed in \cref{prop:CI_oracle} and $\alpha \ldef \frac{\gamma}{4}$.

Since the theoretical analysis of \cref{thm:constant} and \cref{thm:constant_adv} do not require an non-increasing (with respect to time step $t$) sampling region, i.e., $\crl{x \in \cX: g_t(x) = 1}$, we only need to approximate the confidence intervals at $\frac{\gamma}{4}$ level.
This slightly save the computational complexity 
compared to \cref{thm:epoch_efficient}, which approximates the confidence interval at $\frac{\gamma}{4 \ceil{\log T}}$ level. The rest of the analysis of computational complexity follows similar steps in the proof of \cref{thm:epoch_efficient}.
\end{proof}

\subsection{Supporting lemmas}
\label{app:constant_lms}

Consider a sequence of random variables $\prn{Z_t}_{t \in \N}$ adapted to filtration ${\wb \mfF_t}$. 
We assume that $\E \sq*{\exp(\lambda Z_t)} < \infty$ for all $\lambda$. Denote $\mu_t \ldef \E \sq*{ Z_t  \mid \wb \mfF_{t-1} }$ 
and
$$ \psi_t(\lambda) \ldef \log  \E \sq*{ \exp( \lambda \cdot \paren*{ Z_t - \mu_t } )  \mid \wb \mfF_{t-1} } 
.$$ 

\begin{lemma}[\citet{russo2013eluder}]
	\label{lm:opt_stopping}
With notations defined above. For any $\lambda \geq 0$ and $\delta > 0$, we have 
\begin{align}
    \P \prn*{ \forall \tau \in \N, \sum_{t = 1}^{\tau} \lambda Z_t \leq \sum_{t=1}^{\tau} \prn*{ \lambda \mu_t + \psi_t(\lambda) } + \log \prn*{\frac{1}{\delta}} } \geq 1 -\delta.
\end{align}
\end{lemma}

\begin{lemma}
    \label{lm:expected_sq_loss_opt}
   Fix any $\delta \in (0,1)$.  For any $\tau\in [T]$, with probability at least $1 - \delta $, we have 
   \begin{align*}
   	\sum_{t = 1}^{\tau} M_t(f) \leq \sum_{t=1}^{\tau} \frac{3}{2} \E_t \brk{M_t(f)} + 
	C_\delta,
   \end{align*}
   and
   \begin{align*}
       \sum_{t = 1}^{\tau} \E_t \sq{ M_t(f)}  \leq 2 \sum_{t = 1}^{\tau} M_t(f) + C_\delta,
   \end{align*}
   where $C_\delta \ldef  4 \log \prn*{ \frac{2 \abs{\cF} }{\delta}}$.
\end{lemma}
\begin{proof}
	Fix any $f \in \cF$. We take $Z_t = M_t(f)\ldef Q_t \prn{ \prn{f(x_t) - y_t}^2 - \prn{f^\star(x_t) - y_t}^2}$ in \cref{lm:opt_stopping}. 
	We can rewrite 
	$$Z_t = Q_t \prn*{ \prn{f(x_t) - f^{\star}(x_t) }^2 + 2 \prn{f(x_t) - f^{\star}(x_t)} \eps_t},$$ 
	where we use the notation $\eps_t \ldef f^{\star}(x_t) - y_t$. Since $\E_t [\eps_t] = 0$ and $\E_t \brk{\exp(\lambda \eps_t) \mid \wb \mfF_{t-1} } \leq \exp(\frac{\lambda^2}{2})$ by Hoeffding Lemma, we have 
	\begin{align*}
		\mu_t  = \E_t \brk{Z_t}  = Q_t \prn*{ f(x_t) - f^{\star}(x_t)}^2 ,
	\end{align*}
	and 
	\begin{align*}
		\psi_t(\lambda) & = \log  \E \sq*{ \exp( \lambda \cdot \paren*{ Z_t - \mu_t } )  \mid \wb \mfF_{t-1} }\\
		& = \log \E_t \brk*{ \exp \prn*{2 \lambda Q_t \prn*{f(x_t) - f^{\star}(x_t) \cdot \eps_t }}} \\
		& \leq { \frac{\prn*{2 \lambda Q_t ( f(x_t) - f^{\star}(x_t) }^2 }{2}} \\
		& = {2 \lambda^2 \mu_t},
	\end{align*}
	where the last line comes from the fact that $Q_t \in \crl{0,1}$. 
	Plugging these results into \cref{lm:opt_stopping} with $\lambda = 1 / 4$ leads to 
	\begin{align*}
		\sum_{t=1}^{\tau} M_t (f) \leq \sum_{t=1}^{\tau} \frac{3}{2} \E_t \brk*{ M_t(f)} + 4 \log \delta^{-1}. 
	\end{align*}
    	Following the same procedures above with $Z_t = - M_t(f)$ and $\lambda = 1/ 4$ leads to 
	\begin{align*}
		\sum_{t=1}^{\tau} \E_t \brk*{ M_t(f)} \leq 2 \sum_{t=1}^{\tau} M_t (f) + 4 \log \delta^{-1}. 
	\end{align*}
	The final guarantees come from taking a union abound over $f \in \cF$ and splitting the probability for both directions.
\end{proof}
We use $\cE$ to denote the good event considered in \cref{lm:expected_sq_loss_opt}, we use it through out the rest of this section.

\begin{lemma}
\label{lm:set_f_eluder}
With probability at least $1-\delta$, the followings hold true:
\begin{enumerate}
	\item $f^\star \in \cF_t$ for any $t \in [T]$.
	\item $\sum_{t=1}^{\tau-1} \E_t \brk{M_t(f)} \leq 2   C_\delta$ for any $f \in \cF_{\tau}$. 
\end{enumerate}
\end{lemma}
\begin{proof}
	The first statement immediately follows from \cref{lm:expected_sq_loss_opt} (the second inequality) and the fact that $\beta \ldef C_{\delta}/{2}$ in \cref{alg:eluder}.

For any $f \in \cF_{\tau}$, we have 
\begin{align}
\sum_{t=1}^{\tau-1} \E_t \brk{M_t(f)} & \leq 2\sum_{t=1}^{\tau-1} Q_t \prn*{ (f(x_t) - y_t)^2 - ( f^{\star}(x_t) - y_t)^2 } + C_\delta\nonumber\\
    & \leq 2\sum_{t=1}^{\tau-1} Q_t \prn*{ (f(x_t) - y_t)^2 - ( \wh f_\tau(x_t) - y_t)^2 }+ C_\delta \nonumber \\
    & \leq 2 C_\delta, \label{eq:squared_diff_f_star}
\end{align}
where the first line comes from \cref{lm:expected_sq_loss_opt}, the second line comes from the fact that $\wh f_\tau$ is the minimize among $\cF_\tau$, and the third line comes from the fact that $f \in \cF_\tau$ and $2 \beta = C_\delta$. 
\end{proof}

\begin{lemma}
\label{lm:conf_width_eluder}
 For any $\zeta > 0$, with probability $1$, we have 
\begin{align*}
	\sum_{t = 1}^{T} \ind \prn*{Q_t = 1} \cdot \ind \prn*{ w_t > \zeta } < \prn*{ \frac{16\beta}{\zeta^2} + 1 } \cdot \mfe_{f^{\star}}(\cF, \zeta/2).
\end{align*}
\end{lemma}
\begin{remark}
	Similar upper bound has been established in the contextual bandit settings for $\sum_{t=1}^{T} \ind(w_t > \zeta)$ \citep{russo2013eluder, foster2020instance}. 
	 We develop our results with an additional $\ind(Q_t=1)$ term to account for selective querying in active learning.
\end{remark}
\begin{proof}
We give some definitions first. We say that $x$ is $\zeta$-independent of a sequence $x_1, \dots, x_{\tau}$ if there exists a $f \in \cF$ such that $\abs*{f(x) - f^\star(x)} > \zeta$ and $\sum_{i \leq {\tau}} \paren{f(x_i) - f^\star(x_i)}^2 \leq \zeta^2$. We say that $x$ is $\zeta$-dependent of $x_1, \dots, x_{\tau}$ if we have $\abs*{f(x) -f^\star(x)} \leq \zeta$ for all $f \in \cF$ such that $\sum_{i \leq {\tau}} \paren{f(x_i) - f^\star(x_i)}^2 \leq \zeta^2$. 

For any $t \in [T]$, and we denote $\cS_{t} = \curly*{x_i: Q_i = g_i(x_i)=1 , i \in [t]}$ as the \emph{queried} data points up to time step $t$. We assume that $\abs*{\cS_t} = \tau$ and denote $\cS_{t} = (x_{g(1)}, \dots, x_{g(\tau)})$, where $g(i)$ represents the time step where the $i$-th \emph{queried} data point is queried.

\textbf{Claim 1.} For any $j \in [\tau]$, $x_{g(j)}$ is $\frac{\zeta}{2}$-dependent on at most $\frac{16 \beta}{\zeta^2}$ disjoint subsequences of $x_{g(1)}, \dots, x_{g(j-1)}$.

For any $x_{g(j)} \in \cS_t$, recall that
$$w_{g(j)} = \ucb_{g(j)} - \lcb_{g(j)} = \max_{f, f^\prime \in \cF_{g(j)}} \abs*{f(x_t) - f^\prime(x_t)}.$$ 
If $w_{g(j)} > \zeta$, there must exists a $f \in \cF_{g(j)}$ such that $\abs*{f(x_{g(j)}) - f^\star(x_{g(j)})} > \frac{\zeta}{2}$. Focus on this specific $f \in \cF_{g(j)} \subseteq \cF$. 
If $x_{g(j)}$ is $\frac{\zeta}{2}$-dependent on a subsequence $x_{g(i_1)}, \dots, x_{g(i_{m})}$ (of $x_{g(1)}, \dots, x_{g(j-1)}$), we must have 
\begin{align*}
    \sum_{k \leq m} \paren{ f(x_{g(i_k)}) - f^\star(x_{g(i_k)}) }^2 > \frac{\zeta^2}{4}.
\end{align*}
Suppose $x_{g(j)}$ is $\frac{\zeta}{2}$-dependent on $K$ \emph{disjoint} subsequences of $x_{g(1)}, \dots, x_{g(j-1)}$, according to \cref{lm:set_f_eluder}, we must have
\begin{align*}
    K \cdot \frac{\zeta^2}{4} < \sum_{i < j} \paren{f(x_{g(i)}) - f^\star(x_{g(i)})}^2 = \sum_{k < g(j)} Q_k \paren{f(x_k) - f^\star(x_k)}^2  \leq 4 \beta,
\end{align*}
which implies that $K < \frac{16 \beta}{\zeta^2}$.

\textbf{Claim 2.} Denote $d \ldef \check \mfe_{f^\star}(\cF, \zeta/2) \geq 1$ and $K \ldef \floor*{\frac{\tau-1}{d}}$. There must exists a $j \in [\tau]$ such that $x_{g(j)}$ is $\frac{\zeta}{2}$-dependent on at least $K$ disjoint subsequences of $x_{g(1)}, \dots, x_{g(j-1)}$.

We initialize $K$ subsequences $\cC_i = \crl{x_{g(i)}}$. If $x_{g(K+1)}$ is $\frac{\zeta}{2}$-dependent on each $\cC_i$, we are done. If not, select a subsequence $\cC_i$ such that $x_{g(K+1)}$ is $\frac{\zeta}{2}$-independent of and add $x_{g(K+1)}$ into this subsequence. Repeat this procedure with $j > K+1$ until $x_{g(j)}$ is $\frac{\zeta}{2}$-dependent of all $\cC_i$ or $j = \tau$. In the first case we prove the claim. In the later case, we have $\sum_{i \leq K} \abs*{\cC_i} = \tau - 1 \geq Kd$. Since $\abs*{\cC_i} \leq d$ by the construction of $\cC_i$ and the definition of $\check \mfe_{f^\star}(\cF, \zeta/2)$, we must have $\abs*{\cC_i} = d$ for all $i \in [K]$. As a result, $x_{g(\tau)}$ must be $\frac{\zeta}{2}$-dependent of all $\cC_i$.

It's easy to check that $\floor*{\frac{\tau - 1}{d}} \geq \frac{\tau}{d} - 1$. Combining Claim 1 and 2, we have 
\begin{align*}
    \frac{\tau}{d} -1 \leq \floor*{\frac{\tau - 1}{d}} \leq K < \frac{16 \beta}{\zeta^2}.
\end{align*}
Rearranging leads to the desired result.
\end{proof}
The following \cref{lm:regret_no_query_constant} is a restatement of \cref{lm:regret_no_query} in the regret minimization setting.
\begin{lemma}
\label{lm:regret_no_query_constant}
If $Q_t = 0$, we have 
    $\E \sq*{\ell_t(a_t) - \ell_t(a_t^\star) \mid \wb \mfF_{t-1}} \leq 0$.
\end{lemma}
\begin{proof}
	Recall we have $a_t = \wh h_t (x_t)$. We then have  
\begin{align*}
    & \E  \sq*{\ell_t(a_t) - \ell_t(a_t^\star) \mid \wb \mfF_{t-1} } \\
    & = \P_{y_t \mid x_t} \prn[\big]{y_t \neq \widehat h_t(x_t)} \cdot \ind \prn[\big]{ \widehat h_t(x_t) \neq \bot} + \prn[\big]{{1}/{2} - \gamma} \cdot \1 \prn[\big]{\widehat h_t(x_t) = \bot} - \P_{y_t\mid x_t} \prn[\big]{ y_t \neq h^\star(x_t) }\\
    & = \ind \prn[\big]{ \widehat h_t(x_t) \neq \bot} \cdot \prn[\big]{\P_{y_t\mid x_t} \prn[\big]{y_t \neq \widehat h_t(x_t)} -  \P_{y_t\mid x_t} \prn[\big]{ y_t \neq h^\star(x_t) }} \\
    & \quad + \ind \prn[\big]{ \widehat h_t(x_t) = \bot} \cdot \prn[\big]{ \prn[\big]{{1}/{2} - \gamma}  -  \P_{y_t \mid x_t} \prn[\big]{ y_t \neq h^\star(x_t) }}.
\end{align*}
We now analyze the event $\curly*{Q_t = 0}$ in two cases. 

\textbf{Case 1: ${\widehat h_t(x_t) = \bot} $.} 

Since $\eta(x_t) = f^{\star}(x_t) \in [\lcb_t, \ucb_t]$, we further know that $\eta(x_t) \in \sq{ \frac{1}{2} - \gamma, \frac{1}{2} + \gamma }$ and thus $\P_{y_t\mid x_t} \prn[\big]{ y_t \neq h^\star(x_t) } \geq \frac{1}{2} - \gamma$. As a result, we have $\E \sq*{\ell_t(a_t) - \ell_t(a_t^\star) \mid \wb \mfF_{t-1} } \leq 0$.

\textbf{Case 2: ${\widehat h_t(x_t) \neq \bot}$ but ${\frac{1}{2} \notin (\lcb_t, \ucb_t)} $.} 

In this case, we know that $\widehat h_t (x_t) = h^\star(x_t)$ whenever $\eta(x_t) \in [\lcb_t, \ucb_t]$. As a result, we have \\
  $\E \sq*{\ell_t(a_t) - \ell_t(a_t^\star) \mid \wb \mfF_{t-1} } = 0$.
\end{proof}
\begin{lemma}
\label{lm:regret_eluder_cond}
 Assume $\mu(x_t) \in [\lcb_t, \ucb_t]$ and $f^{\star}$ is not eliminated across all $t \in [T]$. We have 
\begin{align*}
	\sum_{t=1}^{T} \E \brk{\ell_t(a_t) - \ell_t(a_t^\star) \mid \wb \mfF_{t-1} } \leq 
	\frac{17 \sqrt{2} \beta}{\gamma} \cdot \mfe_{f^{\star}}(\cF, \gamma / 2).
\end{align*}
\end{lemma}
\begin{proof}
	\cref{lm:regret_no_query_constant} shows that non-positive conditional regret is incurred at whenever $Q_t = 0$, we then have 
	\begin{align*}
	\sum_{t=1}^{T} \E \brk{\ell_t(a_t) - \ell_t(a_t^\star) \mid \wb \mfF_{t-1} } & 
	\leq \sum_{t=1}^{T} \ind(Q_t =1) \E \brk*{ \ell_t(a_t) - \ell_t(a_t^{\star}) \mid \wb \mfF_{t-1}}\\
	& \leq \sum_{t=1}^{T} \ind(Q_t = 1) \cdot \ind (w_t > \gamma) \cdot \abs*{ 2 f^{\star}(x_t) - 1}\\
	& \leq \sum_{t=1}^{T}\ind(Q_t = 1) \cdot \ind (w_t > \gamma) \cdot 2 w_t,
	\end{align*}
  where the second line comes from the fact that, under the event $\crl{Q_t = 1}$, we have $w_t > \gamma$ (using a similar analysis as in \cref{lm:query_implies_width}) and $\E \brk*{ \ell_t(a_t) - \ell_t(a_t^{\star}) \mid \wb \mfF_{t-1}} \leq \abs{2 f^\star(x_t) - 1}$ (since $a_t \neq \bot$), the last line comes from the fact that $ \abs{f^{\star}(x_t) - \frac{1}{2}} \leq  w_t$ whenever $f^{\star}$ is not eliminated and $Q_t = 1$.
	We can directly apply $w_t \leq 1$ and \cref{lm:conf_width_eluder} to bound the above terms by 
	$\wt O(\frac{ \mfe_{f^{\star}}(\cF, \gamma / 2)}{\gamma^2})$, which has slightly worse dependence on $\gamma$. Following \citet{foster2020instance}, we take a slightly tighter analysis below. 

	Let $\cS_T \ldef \crl{ x_i: Q_i =1, i \in [T] } $ denote the set of queried data points. Suppose $\abs{\cS_T} = \tau$. 
	Let $i_1, \ldots, i_\tau$ be a reordering of indices within $\cS_T$ such that $w_{i_1}(x_{i_1}) \geq w_{i_2}(x_{i_2}) \geq \ldots \geq w_{i_\tau} (x_{i_\tau})$. 
	Consider any index $t \in [\tau]$ such that $w_{i_t} (x_{i_t}) \geq \gamma$. For any $\zeta \geq \gamma$, \cref{lm:conf_width_eluder} implies that 
	\begin{align}
	t \leq \sum_{t=1}^{T} \ind(Q_t = 1) \cdot \ind ( w_t(x_t) > \zeta) \leq \frac{17 \beta}{\zeta^2} \cdot \mfe_{f^{\star}}\prn*{\cF, {\zeta}/{2}} \leq \frac{17 \beta}{\zeta^2} \cdot \mfe_{f^{\star}}\prn*{\cF, {\gamma}/{2}} . \label{eq:regret_eluder_cond}
	\end{align}
  Taking $\zeta = w_{i_t}(x_{i_t})$ in \cref{eq:regret_eluder_cond} leads to the following inequality on $w_{i_t} (x_{i_t})$: 
$$
w_{i_t}(x_{i_t}) \leq \sqrt{\frac{17 \beta \cdot \mfe_{f^{\star}}(\cF, \gamma/ 2)}{t}}
.$$ 
  Taking $\zeta = \gamma$ in \cref{eq:regret_eluder_cond} leads to the following inequality on $\tau$:
$$
\tau \leq \frac{17 \beta}{\gamma^2} \cdot \mfe_{f^{\star}}(\cF, \gamma / 2)
.$$ 
We then have 
\begin{align*}
	\sum_{t=1}^{T}\ind(Q_t = 1) \cdot \ind (w_t > \gamma) \cdot 2 w_t & = 
	\sum_{t=1}^{\tau}  \ind(w_{i_t} > \gamma) \cdot 2 w_{i_t}(x_{i_t}) \\
	& \leq 2 \,\sum_{t=1}^{\tau} \sqrt{\frac{17 \beta \cdot \mfe_{f^{\star}}(\cF, \gamma/ 2)}{t}}\\
	& \leq \sqrt{34 \beta \cdot \mfe_{f^{\star}}(\cF, \gamma / 2) \cdot \tau}	\\
	& \leq \frac{17 \sqrt{2} \beta}{\gamma} \cdot \mfe_{f^{\star}}(\cF, \gamma / 2).
\end{align*}
\end{proof}

\begin{lemma}
  \label{lm:wh_a_ineq}
  We have $0 \leq \E \brk{\ell_t(a_t) - \ell_t(\wh a_t^\star) \mid \wb \mfF_{t-1}} \leq 1$ and $\E \brk{\ell_t(\wh a_t^\star) - \ell_t(a_t^\star) \mid \wb \mfF_{t-1}} \leq 0$.
\end{lemma}
\begin{proof}
  By construction, we have $\wh a_t^\star = \bot$ if $a_t = \bot$, and $\wh a_t^\star = a_t^\star$ otherwise.
  Similar to the analysis in \cref{lm:regret_no_query_constant}, we have
\begin{align*}
     \E  \sq*{\ell_t(a_t) - \ell_t(\wh a_t^\star) \mid \wb \mfF_{t-1} } 
     = \ind \prn[\big]{ \widehat h_t(x_t) \neq \bot} \cdot \prn[\big]{\P_{y_t\mid x_t} \prn[\big]{y_t \neq \widehat h_t(x_t)} -  \P_{y_t\mid x_t} \prn[\big]{ y_t \neq h^\star(x_t) }}, 
\end{align*}
and 
\begin{align*}
     \E  \sq*{\ell_t(\wh a^\star_t) - \ell_t(a_t^\star) \mid \wb \mfF_{t-1} } 
     = \ind \prn[\big]{ \widehat h_t(x_t) = \bot} \cdot \prn[\big]{ \prn[\big]{{1}/{2} - \gamma}  -  \P_{y_t \mid x_t} \prn[\big]{ y_t \neq h^\star(x_t)}}.
\end{align*}
  The statement $0 \leq \E \brk{\ell_t(a_t) - \ell_t(\wh a_t^\star) \mid \wb \mfF_{t-1}} \leq 1$ follows from the fact that $0 \leq {\P_{y_t\mid x_t} \prn[\big]{y_t \neq \widehat h_t(x_t)} -  \P_{y_t\mid x_t} \prn[\big]{ y_t \neq h^\star(x_t) }} \leq 1$ when $\wh h_t(x_t) \neq \bot$.

  Similar to the analysis in \cref{lm:regret_no_query_constant}, we have $\P_{y_t\mid x_t} \prn[\big]{ y_t \neq h^\star(x_t) } \geq \frac{1}{2} - \gamma$ when $\wh h_t(x_t) = \bot$. This leads to $\E \brk{\ell_t(\wh a_t^\star) - \ell_t(a_t^\star) \mid \wb \mfF_{t-1}} \leq 0$.
\end{proof}

\begin{lemma}
\label{lm:regret_eluder_cond_wh_a}
 Assume $\mu(x_t) \in [\lcb_t, \ucb_t]$ and $f^{\star}$ is not eliminated across all $t \in [T]$. We have 
\begin{align*}
  \sum_{t=1}^{T} \E \brk{\ell_t(a_t) - \ell_t(\wh a_t^\star) \mid \wb \mfF_{t-1} } \leq 
	\frac{17 \sqrt{2} \beta}{\gamma} \cdot \mfe_{f^{\star}}(\cF, \gamma / 2).
\end{align*}
\end{lemma}

\begin{proof}
  We first consider the event $\crl{Q_t = 0}$. We have 
\begin{align*}
     \E  \sq*{\ell_t(a_t) - \ell_t(\wh a_t^\star) \mid \wb \mfF_{t-1} } 
     = \ind \prn[\big]{ \widehat h_t(x_t) \neq \bot} \cdot \prn[\big]{\P_{y_t\mid x_t} \prn[\big]{y_t \neq \widehat h_t(x_t)} -  \P_{y_t\mid x_t} \prn[\big]{ y_t \neq h^\star(x_t) }}. 
\end{align*}
  When $\wh h_t(x_t) \neq \bot$ and $Q_t = 0$, we must have $\frac{1}{2} \notin \prn{\lcb_t, \ucb_t}$. We then have $\wh h_t(x_t) = h^\star(x_t)$, which leads to $\E  \sq*{\ell_t(a_t) - \ell_t(\wh a_t^\star) \mid \wb \mfF_{t-1} }=0$.

  With the above results on the event $\crl{Q_t = 0}$, the rest of the analysis are the same as the analysis as in \cref{lm:regret_eluder_cond} since $\wh a_t^\star = a_t^\star$ under event $\crl{Q_t = 1}$.
\end{proof}
	
\section{Omitted details for \cref{sec:misspecified}}
\label{app:mis}
\subsection{Algorithm and main results}
\label{app:mis_alg}
\begin{algorithm}[]
	\caption{Efficient Active Learning with Abstention under Misspecification}
	\label{alg:mis} 
	\renewcommand{\algorithmicrequire}{\textbf{Input:}}
	\renewcommand{\algorithmicensure}{\textbf{Output:}}
	\newcommand{\algorithmicbreak}{\textbf{break}}
    \newcommand{\BREAK}{\STATE \algorithmicbreak}
	\begin{algorithmic}[1]
		\REQUIRE Accuracy level $\epsilon > 0$, abstention parameter $\gamma \in (\eps, 1/2)$ and confidence level $\delta \in (0, 1)$.
		\STATE Define $T \ldef \frac{\pseud(\cF)}{\eps \, \gamma}$, $M \ldef \ceil{\log_2 T}$ and $C_\delta \ldef O \prn{\pseud(\cF) \cdot \log(T /\delta)}$.
		\STATE Define $\tau_m \ldef 2^m$ for $m\geq1$, $\tau_0 = 0$ and $\beta_m \ldef \prn*{M-m+1}\cdot \prn*{2\eps^2 \tau_{M-1} +  2C_\delta  }$.  
		\FOR{epoch $m = 1, 2, \dots, M$}
		\STATE Get $\widehat f_m \ldef \argmin_{f \in \cF} \sum_{t=1}^{\tau_{m-1}} Q_t \paren{f(x_t) - y_t}^2 $.\\
		\hfill \algcommentlight{We use $Q_t \in \crl{0,1}$ to indicate whether the label of  $x_t$ is queried.}
		\STATE (Implicitly) Construct active set of regression function $\cF_m \subseteq \cF$ as 
		\begin{align*}
		    \cF_m \ldef \crl*{ f \in \cF:  \sum_{t = 1}^{\tau_{m-1}} Q_t \prn*{f(x_t) - y_t}^2 \leq \sum_{t = 1}^{\tau_{m-1}} Q_t \paren{\widehat f_m(x_t) - y_t}^2 + \beta_m }. 
		\end{align*}
		\STATE Construct classifier $\wh h_m: \cX \rightarrow \crl{0,1,\bot}$ as
		\begin{align*}
			\wh h_m (x) \ldef 
			\begin{cases}
				\bot, & \text{ if } \brk { \lcb(x;\cF_m), \ucb(x;\cF_m) } \subseteq 
				\brk*{ \frac{1}{2} - \gamma, \frac{1}{2} + \gamma}; \\
        \ind(\wh f_m(x) \geq \frac{1}{2} ) , & \text{ o.w. }
			\end{cases}
		\end{align*}
		and query function $g_m: \cX \rightarrow \crl{0,1}$ as
		\begin{align*}
		g_m(x) \ldef\ind \prn*{ \frac{1}{2} \in \prn{\lcb(x;\cF_m) , \ucb(x;\cF_m) } } \cdot
		\ind \prn{\wh h_m(x) \neq \bot} .
		\end{align*}
		\IF{epoch $m=M$ }
		\STATE \textbf{Return} classifier $\wh h_{M} $.
		\ENDIF
		\FOR{time $t = \tau_{m-1} + 1 ,\ldots , \tau_{m} $} 
		\STATE Observe $x_t \sim \cD_{\cX}$. Set $Q_t \ldef g_m(x_t)$. 
		\IF{$Q_t = 1$}
		\STATE Query the label $y_t$ of $x_t$.
		\ENDIF
		\ENDFOR
		\ENDFOR
	\end{algorithmic}
\end{algorithm}

\cref{alg:mis} achieves the guarantees stated in \cref{thm:mis}. 
\cref{thm:mis} is proved based on supporting lemmas derived in \cref{app:mis_lms}. 
Note that, under the condition $\kappa \leq \eps$, we still compete against the Bayes classifier $h^{\star} = h_{f^{\star}}$ in the analysis of Chow's excess error \cref{eq:chow_error}.

\thmMis*

\begin{proof}
	We analyze under the good event $\cE$ defined in \cref{lm:expected_sq_loss_pseudo}, which holds with probability at least $1-\delta$. Note that all supporting lemmas stated in \cref{app:mis_lms} hold true under this event.

We analyze the Chow's excess error of $\wh h_m$, which is measurable with respect to $\mfF_{\tau_{m-1}}$. 
For any $x \in \cX$, if $g_m(x) = 0$, 
\cref{lm:regret_no_query_mis} implies that $\exc_{\gamma}(\wh h_m ;x) \leq 2 \kappa $. 
If $g_m(x)= 1$, we know that $\wh h_m(x) \neq \bot$ and $\frac{1}{2} \in (\lcb(x;\cF_m),\ucb(x;\cF_m))$. 
Since $\wb f \in \cF_m$ by \cref{lm:set_f_mis} and $\sup_{x \in \cX} \abs{ \wb f(x) - f^{\star}(x)} \leq \kappa$ by assumption. 
The error incurred in this case is upper bounded by 
\begin{align*}
	\exc_{\gamma}(\wh h_m; x) 
	& \leq 2 \abs{ f^{\star}(x)- 1 /2}\\
	& \leq 2\kappa + 2 \abs{ \wb f(x)- 1 /2}\\
	& \leq 2\kappa + 2 w(x;\cF_m).
\end{align*}
Combining these two cases together, we have 
\begin{align*}
	\exc_{\gamma}( \wh h_m) \leq 2 \kappa +  2 \E_{x \sim \cD_\cX} \brk{ \ind(g_m(x) = 1) \cdot w(x;\cF_m)}.	
\end{align*}
Take $m=M$ and apply \cref{lm:per_round_regret_dis_coeff_mis} leads to the following guarantee.
\begin{align*}
	\exc_{\gamma}( \wh h_M)
	& \leq  2 \kappa + { \frac{ 72 \beta_M}{\tau_{M-1} \gamma} \cdot \theta^{\val}_{\wb  f}\prn*{\cF, \gamma/2, \sqrt{\beta_M/ \tau_{M-1}}}}\\
	& \leq 2 \kappa +  O \prn*{ \frac{\eps^2}{\gamma} + \frac{ \pseud(\cF) \cdot \log ( T / \delta)}{T \, \gamma}} \cdot \theta^{\val}_{ \wb f}\prn*{\cF, \gamma/2, \sqrt{C_\delta/T}} \\
	& = O \prn*{ \eps \cdot \wb \theta \cdot \log \prn*{ \frac{\pseud(\cF) }{\eps \, \gamma \, \delta}}},
\end{align*}
where we 
take $\wb \theta \ldef \sup_{\iota > 0}\theta^{\val}_{\wb f} (\cF, \gamma / 2, \iota)$ as an upper bound of $\theta^{\val}_{\wb f} (\cF, \gamma / 2, \sqrt{C_\delta / T})$,
and use the fact that $T = {\frac{\pseud(\cF) }{\eps \, \gamma}}$ and the assumptions that $\kappa \leq \eps < \gamma$.

We now analyze the label complexity (note that the sampling process of \cref{alg:mis} stops at time $t = \tau_{M-1}$).
Note that $\E \brk{\ind(Q_t = 1) \mid \mfF_{t-1}} = \E_{x\sim\cD_\cX} \brk{ \ind(g_m(x) = 1) }$ for any epoch $m \geq 2$ and time step $t$ within epoch $m$. 
Combining \cref{lm:martingale_two_sides} with \cref{lm:conf_width_dis_coeff_mis} leads to
    \begin{align*}
        \sum_{t=1}^{\tau_{M-1}} \ind(Q_t = 1) & \leq \frac{3}{2} \sum_{t=1}^{\tau_{M-1}} \E \sq{\ind(Q_t = 1) \mid \mfF_{t-1}} + 4 \log \delta^{-1}\\
        & \leq 3 + \frac{3}{2}\sum_{m=2}^{M-1}\frac{(\tau_m - \tau_{m-1}) \cdot 36 \beta_m}{{\tau_{m-1}} \, \gamma^2} \cdot \theta^{\val}_{\wb f}\prn*{\cF, \gamma/2, \sqrt{\beta_m/\tau_{m-1}}}  + 4 \log \delta^{-1} \\
        & \leq 3 + 48 \sum_{m=2}^{M-1}\frac{\beta_m}{ \gamma^2} \cdot \theta^{\val}_{\wb f}\prn*{\cF, \gamma/2, \sqrt{\beta_m/\tau_{m-1}}}  + 4 \log \delta^{-1} \\
	& \leq 3 + 4 \log \delta^{-1} + O \prn*{ \frac{M^2\cdot \eps^2 \cdot T}{\gamma^2} + \frac{M^2 \cdot C_\delta }{\gamma^2}}
	\cdot \theta^{\val}_{\wb f}\prn*{\cF, \gamma/2, \sqrt{C_\delta/T }} \\ 
	& = O \prn*{ \frac{\wb \theta \, \pseud(\cF)}{\gamma^2} \cdot \prn*{\log \prn*{ \frac{\pseud(\cF)}{\eps\, \gamma }}}^{2} \cdot 
	\log \prn*{ \frac{ \pseud(\cF) }{\eps \,\gamma \,  \delta}}} 
    \end{align*}
    with probability at least $1-2\delta$ (due to an additional application of \cref{lm:martingale_two_sides}); where we use the fact that $T = {\frac{\pseud(\cF) }{\eps \, \gamma}}$ and the assumptions that $\kappa \leq \eps < \gamma$ as before.
\end{proof}

\begin{restatable}{theorem}{thmMisEfficient}
	\label{thm:mis_efficient}
	\cref{alg:mis} can be efficiently implemented via the regression oracle and enjoys the same theoretical guarantees stated in \cref{thm:mis}.
	The number of oracle calls needed is $\wt O(\frac{\pseud(\cF) }{\eps \, \gamma^{3}})$ for a general set of regression functions $\cF$, and $\wt O(\frac{\pseud(\cF)}{\eps \, \gamma})$ when $\cF$ is convex and closed under pointwise convergence.
	The per-example inference time of the learned $\wh h_{M}$ is $\wt O ( \frac{1}{\gamma^2} \log^2 \prn{\frac{\pseud(\cF)}{\eps }})$ for general $\cF$, and $\wt O ( \log \frac{1}{\gamma}) $ when $\cF$ is convex and closed under pointwise convergence.
\end{restatable}

\begin{proof}
Note that classifier $\wh h_m$ and query function $q_m$ in \cref{alg:mis} are constructed in the way as the ones in \cref{alg:epoch},
Thus, \cref{alg:mis} can be efficiently implemented in the same way as discussed in \cref{thm:epoch_efficient}, and enjoys the same per-round computational complexities.
The total computational complexity is then achieved by multiplying the per-round computational complexity by $T = {\frac{\pseud(\cF) }{\eps \, \gamma}}$.
\end{proof}

\subsection{Discussion on $\kappa \leq \eps$}
\label{app:mis_partial}

We provide guarantees (in \cref{thm:mis}) when $\kappa \leq \eps$ since the learned classifier suffers from an additive  $\kappa$ term in the excess error, as shown in the proof of \cref{thm:mis}.
We next give preliminary discussions on this issue by relating active learning with to a (specific) regret minimization problem and connecting to existing lower bound in the literature.
More specifically, we consider the perspective and notations discussed in \cref{app:constant_regret}.
Fix any epoch $m \geq 2$ and time step $t$ within epoch $m$. 
We have 
\begin{align*}
    \reg_t = \E \sq{\ell_t(a_t) - \ell_t(a^\star_t)  \mid  \mfF_{t-1}} = \err_{\gamma}(\wh h_m) - \err(h^\star) = \exc_{\gamma}(\wh h_m) = \wt O \prn*{ \kappa + \frac{\wb \theta}{2^m \, \gamma} },
\end{align*}
where the bound comes from similar analysis as in the proof of \cref{thm:mis}. Summing the instantaneous regret over $T$ rounds, we have 
\begin{align*}
    \reg(T) & = \sum_{t=1}^T \reg_t\\
    & \leq 2 + \sum_{m=2}^M (\tau_m - \tau_{m-1}) \cdot \exc_{\gamma}(\wh h_m)\\
    & \leq \wt O \prn*{\kappa \cdot T + \frac{\wb \theta}{\gamma}}.
\end{align*}
The above bound indicates an additive regret term scales as $\kappa \cdot T$. On the other hand, 
it is known that an additive $\kappa \cdot T$ regret is in general unavoidable in linear bandits under model misspecification \citep{lattimore2020learning}.
This connection partially explains/justifies why we only provide guarantee for \cref{thm:mis} under $\kappa \leq \eps$.

There are, however, many differences between the two learning problems. We list some distinctions below.
\begin{enumerate}
    \item The regret minimization problem considered in \cref{app:constant_regret} only takes three actions $\cA = \crl{0, 1,\bot}$, yet the lower bound in linear bandits is established with a large action set \citep{lattimore2020learning};
    \item A standard contextual bandit problem will observe loss (with respect to the pulled action) at each step $t \in [T]$, however, the active learning problem will only observe (full) feedback at time steps when a query is issued, i.e., $\crl{t \in [T]: Q_t =1}$.
\end{enumerate}
We leave a comprehensive study of the problem for feature work.

\subsection{Supporting lemmas}
\label{app:mis_lms}

We use the same notations defined in \cref{app:epoch}, except $\wh h_m$, $g_m$ and $\beta_m$ are defined differently.
We adapt the proofs \cref{thm:epoch} (in \cref{app:epoch}) to deal with model misspecification.

Note that although we do not have $f^{\star}\in \cF	$ anymore, one can still define random variables of the form $M_t(f)$, and guarantees in \cref{lm:expected_sq_loss_pseudo} still hold.
We use $\cE$ to denote the good event considered in \cref{lm:expected_sq_loss_pseudo}, we analyze under this event through out the rest of this section.
We also only analyze under the assumption of \cref{thm:mis}, i.e., $\kappa^2 \leq \eps$.
\begin{lemma}
	\label{lm:f_mis}
	Fix any epoch $m \in [M]$. We have 
	$$
	\wh R_m(\wb f) \leq \wh R_m(f^{\star}) + \frac{3 }{2} \cdot \kappa^2 \tau_{m-1} + C_\delta
	,$$ 
	where $C_\delta \ldef 8 \log \prn*{ \frac{\abs{\cF}\cdot T^2}{\delta}}$.
\end{lemma}
\begin{proof}
	From \cref{lm:expected_sq_loss_pseudo} we know that 
	\begin{align*}
		\wh R_m(\wb f) - \wh R_m(f^{\star}) & \leq \sum_{t=1}^{\tau_{m-1}} \frac{3}{2} \cdot
		\E_{t} \brk*{ Q_t\prn*{\wb f(x_t) - f^{\star}(x_t)}^2} + C_\delta \\
		& \leq \frac{3 }{2} \cdot \kappa^2 \tau_{m-1} + C_\delta,
	\end{align*}
	where we use the fact that $\E_t \brk{y_t \mid x_t} = f^{\star}(x_t)$ 
	(and thus $\E_t \brk{ M_t(\wb f)} = \E_{t}\brk{Q_t \prn{\wb f(x_t) - f^{\star}(x_t)}^2}$) on the first line; 
	and use the fact $\sup_{x} \abs{\wb f(x) - f^{\star}(x)} \leq \kappa$ on the second line.
\end{proof}
\begin{lemma}
\label{lm:set_f_mis}
The followings hold true:
\begin{enumerate}
	\item $\wb f\in \cF_m$ for any $m \in [M]$.
	\item $\sum_{t=1}^{\tau_{m-1}} \E_t \brk{M_t(f)} \leq 4 \beta_m $ for any $f \in \cF_m$. 
	\item $ \sum_{t=1}^{\tau_{m-1}} \E \brk{ Q_t(x_t) \prn{ f(x_t) - \wb f(x_t) }^2} \leq 9 \beta_m$ for any $f \in \cF_m$.
	\item $\cF_{m+1} \subseteq \cF_m$ for any $m \in [M-1]$.
\end{enumerate}
\end{lemma}
\begin{proof}
\begin{enumerate}
	\item Fix any epoch $m \in [M]$. By \cref{lm:expected_sq_loss_pseudo}, we have 
$\wh R_m (f^\star) \leq \wh R_m(f) + C_\delta /2 $ for any $f \in \cF$.
Combining this with \cref{lm:f_mis} leads to 
\begin{align*}
	\wh R_m(\wb f) 
	& \leq \wh R_m(f) + \frac{3}{2}\cdot \prn*{ \kappa^2 \tau_{m-1} + C_\delta}\\
	& \leq \wh R_m(f) + \beta_m,
\end{align*}
for any $f \in \cF$, where the second line comes from the definition of $\beta_m$ (recall that we have $\kappa\leq \eps$ by assumption).
We thus have $\wb f \in \cF_m$ for any $m \in [M]$.
\item  Fix any $f \in \cF_m$. With \cref{lm:expected_sq_loss_pseudo}, we have 
	\begin{align*}
		\sum_{t=1}^{\tau_{m-1}} \E_t [M_t(f)] & \leq 2 \sum_{t=1}^{\tau_{m-1}} M_t(f) + C_\delta \\
			& = 2 \wh R_{m }(f) - 2\wh R_{m}(f^{\star}) + C_\delta \\
			& \leq 2 \wh R_{m }(f) - 2\wh R_{m}(\wb f)+ 3 \kappa^2 \tau_{m-1} + 3 C_\delta \\
			& \leq 2 \wh R_{m }(f) - 2\wh R_{m}(\wh f_m)+ 3 \kappa^2 \tau_{m-1} + 3 C_\delta \\
			& \leq 2 \beta_m + 3 \kappa^2 \tau_{m-1} + 3 C_\delta \\
			& \leq 4 \beta_m , 
	\end{align*}
	where the third line comes from \cref{lm:f_mis}; the fourth line comes from the fact that $\wh f_m$ is the minimizer of $\wh R_{m} (\cdot) $; and the fifth line comes from the fact that $f \in \cF_m$.
\item  Fix any $f \in \cF_m$. With \cref{lm:expected_sq_loss_pseudo}, we have 
	\begin{align*}
		\sum_{t=1}^{\tau_{m-1}} \E_t \brk{ Q_t(x_t) \prn{ f(x_t) - \wb f(x_t) }^2}
		& = \sum_{t=1}^{\tau_{m-1}} \E_t \brk{ Q_t(x_t) \prn{ (f(x_t) - f^{\star}(x_t)) + 
		( f^{\star}(x_t) - \wb f(x_t)) }^2} \\
		& \leq 2 \sum_{t=1}^{\tau_{m-1}} \E_t \brk{ Q_t(x_t) \prn{ f(x_t) -  f^{\star}(x_t) }^2} + 2 \tau_{m-1} \kappa^2\\
		&  = 2 \sum_{t=1}^{\tau_{m-1}} \E_t [M_t(f)] + 2 \tau_{m-1} \kappa^2  \\
		& \leq 8 \beta_m + 2 \tau_{m-1} \kappa^2\\
		& \leq 9 \beta_m,
	\end{align*}
	where we use $\prn{a+b}^2 \leq a^2 + b^2$ on the second line; and use statement $2$ on the fourth line.
	\item Fix any $f \in \cF_{m+1}$. We have 
	\begin{align*}
		\wh R_{m} (f) - \wh R_{m} (\wh f_m) & \leq   
		\wh R_{m} (f) - \wh R_{m} (f^{\star}) + \frac{C_\delta}{2}\\
		& = \wh R_{m+1}(f) - \wh R_{m+1}(f^{\star}) 
		- \sum_{t=\tau_{m-1}+1}^{\tau_{m}} M_t(f) + \frac{C_\delta}{2}\\
		& \leq \wh R_{m+1}(f) - \wh R_{m+1} (\wb f) + \frac{3}{2} \kappa^2 \tau_m + C_\delta
		- \sum_{t=\tau_{m-1}+1}^{\tau_{m}} \E_t [M_t(f)] /2 + {C_\delta}\\
		& \leq \wh R_{m+1}(f) - \wh R_{m+1} (\wh f_{m+1}) + \frac{3}{2} \kappa^2 \tau_m + 2C_\delta\\
		& \leq \beta_{m+1} + \frac{3}{2} \kappa^2 \tau_m + 2C_\delta\\
		& \leq \beta_m,
	\end{align*}	
	where the first line comes from \cref{lm:expected_sq_loss_pseudo}; the third line comes from \cref{lm:f_mis} and \cref{lm:expected_sq_loss_pseudo}; the fourth line comes from the fact that $\wh f_{m+1}$ is the minimizer with respect to $\wh R_{m+1}$ and \cref{lm:expected_sq_loss_pseudo}; the last line comes from the definition of $\beta_m$.
\end{enumerate}
\end{proof}
Since the classifier $\wh h_m$ and query function $g_m$ are defined in the same way as in \cref{alg:epoch}, \cref{lm:query_implies_width} holds true for \cref{alg:mis} as well.
As a result of that, \cref{lm:conf_width_dis_coeff} and \cref{lm:per_round_regret_dis_coeff} hold true with minor modifications. We present the modified versions below, whose proofs follow similar steps as in \cref{lm:conf_width_dis_coeff} and \cref{lm:per_round_regret_dis_coeff} but replace $f^{\star}$ with $\wh f$ (and thus using concentration results derived in \cref{lm:set_f_mis}).
\begin{lemma}
    \label{lm:conf_width_dis_coeff_mis}
 Fix any epoch $m \geq 2$. We have 
	\begin{align*}
    \E_{x \sim \cD_\cX} \sq{\ind (g_m(x)= 1)} \leq \frac{36 \beta_m}{{\tau_{m-1}} \, \gamma^2} \cdot \theta^{\val}_{\wb f}\prn*{\cF, \gamma/2, \sqrt{\beta_m/\tau_{m-1}}}.
	\end{align*}
\end{lemma}
\begin{lemma}
    \label{lm:per_round_regret_dis_coeff_mis}
 Fix any epoch $m \geq 2$. We have 
    \begin{align*}
    	\E_{x \sim \cD_\cX} \sq{\ind(g_m(x) = 1)\cdot w(x;\cF_m)} \leq { \frac{36 \beta_m}{\tau_{m-1} \gamma} \cdot \theta^{\val}_{\wb f}\prn*{\cF, \gamma/2, \sqrt{\beta_m/\tau_{m-1}}}}.
    \end{align*}
\end{lemma}
\begin{lemma}
\label{lm:regret_no_query_mis}
Fix any $m \in [M]$. We have $\exc_{\gamma}(\wh h_m ;x) \leq 2 \kappa$ if $g_m(x) = 0$.
\end{lemma}
\begin{proof}
	Recall that
\begin{align*}
	\exc_{\gamma}( \wh h;x) & =  \nonumber
      \ind \prn[\big]{ \widehat h(x) \neq \bot} \cdot \prn[\big]{\P_{y \mid x} \prn[\big]{y \neq \widehat h(x)} -  \P_{y \mid x} \prn[\big]{ y \neq h^\star(x) }} \nonumber \\
    & \quad + \ind \prn[\big]{ \widehat h(x) = \bot} \cdot \prn[\big]{ \prn[\big]{{1}/{2} - \gamma}  -  \P_{y\mid x} \prn[\big]{ y \neq h^\star(x) }} .
\end{align*}
We now analyze the event $\curly*{g_m(x)= 0}$ in two cases. 

\textbf{Case 1: ${\widehat h_m(x) = \bot} $.} 

Since $\wb f(x) \in [\lcb(x;\cF_m), \ucb(x;\cF_m)]$ by \cref{lm:set_f_mis}, we know that $\eta(x) = f^{\star}(x) \in \sq{ \frac{1}{2} - \gamma - \kappa, \frac{1}{2} + \gamma + \kappa}$ and thus $\P_{y} \prn[\big]{ y\neq h^\star(x) } \geq \frac{1}{2} - \gamma-\kappa$. 
As a result, we have $\exc_{\gamma}(\wh h_m;x) \leq \kappa $.

\textbf{Case 2: ${\widehat h_m(x) \neq \bot}$ but ${\frac{1}{2} \notin (\lcb(x;\cF_m), \ucb(x;\cF_m))} $.} 

We clearly have $\exc_{\gamma}(\wh h_m;x) \leq 0$ if $\widehat h_m (x) = h^\star(x)$. 
Now consider the case when $\widehat h_m (x) \neq h^\star(x)$. 
Since $\wb f(x) \in [\lcb(x;\cF_m), \ucb(x;\cF_m)]$ and $\abs{ \wb f(x) - f^{\star}(x)} \leq \kappa$, we must have $\abs*{ f^{\star}(x) - 1 /2  } \leq \kappa$ in that case, which leads to $\exc_{\gamma}(\wh h_m;x) \leq 2 \abs{ f^{\star}(x)-1 / 2} \leq 2 \kappa$.
\end{proof}

\end{document}